\documentclass[11pt]{article}

\usepackage{amsmath}
\usepackage{graphicx}
\usepackage{authblk}
\usepackage{algorithm}
\usepackage{amsthm}
\usepackage{amsfonts}
\usepackage{amssymb}
\usepackage{algpseudocode}
\usepackage{bbm}
\usepackage{bm}
\usepackage{booktabs}
\usepackage{caption}
\usepackage{color}
\usepackage{enumerate}
\usepackage{enumitem}
\usepackage{float}
\usepackage{mathtools} 
\usepackage{hyperref}
\hypersetup{hidelinks}
\usepackage{natbib}
\usepackage{comment}
\usepackage{longtable}
\usepackage{rotating}
\usepackage{xcolor}
\usepackage{fullpage}
\usepackage{tabularx}
\usepackage{url}
\usepackage[section]{placeins}

\newtheorem{lemma}{Lemma}
\newtheorem{corollary}{Corollary}
\newtheorem{proposition}{Proposition}
\newtheorem{theorem}{Theorem}
\theoremstyle{definition}

\newtheorem{assumption}{Assumption}
\newtheorem{remark}{Remark}
\newtheorem{definition}{Definition}

\newcommand{\compactquantilegraphic}[1]{%
  \begingroup
  \setlength{\textwidth}{\dimexpr 2\textwidth/3\relax}%
  #1%
  \endgroup
}

\newcommand{\compactclassificationgraphic}[1]{%
  \begingroup
  \setlength{\textwidth}{\dimexpr 3\textwidth/5\relax}%
  #1%
  \endgroup
}

\algrenewcommand\algorithmicrequire{\textbf{Input:}}
\algrenewcommand\algorithmicensure{\textbf{Output:}}
\DeclareMathOperator*{\argmin}{arg\,min}

\begin{document}
\title{Fair Regression under Approximate Demographic Parity}
\author[1]{Weiwei Wang}
\author[1]{Xianyang Zhang\thanks{Corresponding author: zhangxiany@stat.tamu.edu.}}
\author[1]{Raymond K.~W.~Wong}
\affil[1]{Department of Statistics, Texas A\&M University}
\date{} 

\maketitle

\begin{abstract}
Exact demographic parity requires the same prediction distribution across
sensitive groups. In regression, this can be too restrictive and reduce
predictive accuracy. We instead bound every pairwise Wasserstein distance
between group-specific prediction distributions by a chosen tolerance,
allowing the groups to differ while keeping fairness tunable.

On the real line, the prediction and constrained transport problems have the
same infimum. If the transport problem attains this value and suitable Monge
maps exist, a deterministic transport-form predictor realizes it. Under a
standard one-dimensional condition, the problem becomes a bounded quantile
optimization that is convex when the conditional expected loss is convex in
prediction. Under squared loss with two positively weighted groups, the
$W_1$, $W_2$, and $W_\infty$ constraints yield soft-thresholding, shrinkage,
and clipping. Under the stated conditions, we establish rates for the
known-cost plug-in estimator. When the cost
is unavailable, we describe an observed-loss computational proxy. Synthetic
and California 2018 ACS studies illustrate the tunable fairness--accuracy
trade-off and distinguish calibration feasibility from held-out fairness and
accuracy.
\end{abstract}

\noindent%

\section{Introduction}

Algorithmic fairness asks how predictive models can avoid reinforcing
systematic disadvantages associated with sensitive attributes such as gender or marital status.
This question is especially important in high-stakes settings,
including credit scoring, hiring, criminal justice, and healthcare
\citep{mehrabi2021survey, caton2024fairness}. One widely used group-fairness
criterion is \emph{demographic parity} (DP). It requires the prediction
distribution to be the same in every sensitive group. In probabilistic terms,
the prediction must be independent of group membership.

Optimal transport studies how to move one probability distribution to another
at minimum cost. It provides natural tools for aligning prediction
distributions across groups. In regression under \emph{strict} DP, all groups
must share one prediction distribution. For squared loss, earlier work
identifies the optimal common distribution as a Wasserstein barycenter
\citep{chzhen2020fair}. Recently, \citet{feng2026fair} extend this transport framework
to more general loss functions.

Strict DP has a clear mathematical structure, but it can be overly restrictive in
practice. Requiring identical prediction distributions may reduce predictive
accuracy. This fairness-accuracy trade-off is often called the ``cost of
fairness''. For example, \citet{menon2018cost} study this trade-off in binary classification.

In this paper, we study fair regression under \textbf{Approximate Demographic Parity}.
For $p\in[1,\infty]$, write $W_p$ for the order-$p$ Wasserstein distance, a
distance between probability distributions.
See Section
\ref{sec:wasserstein_constraint} for its formal definition. More specifically, we require the
$W_p$ distance between every pair of group-conditional prediction
distributions to be at most a user-chosen tolerance $\epsilon$. The choice
$\epsilon=0$ recovers strict DP, while larger values allow more flexibility.
Practitioners can therefore vary $\epsilon$ to study the trade-off between
fairness and predictive accuracy.

This relaxation changes the problem in an important way. Strict DP reduces the
fairness problem to choosing one common output distribution. Approximate DP
allows each group to have its own output distribution, provided that every
pair remains close. Since predictions are real-valued, we can compare the
group distributions rank by rank. A quantile function gives the prediction
value at each percentile. Given any proposed set of prediction values, we can
sort the values from low to high without changing their distribution. This
sorting step is called monotone rearrangement. As we show later, it also does
not increase the relevant pairwise distances. Our contribution is to use these
tools to analyze the direct pairwise constraints, derive explicit solutions
where possible, and build numerical methods for the remaining cases.

\subsection*{Our Contributions}

The paper makes three main contributions.

\begin{itemize}
    \item \textbf{Transport and quantile formulation.}
    We first show that the prediction problem and a constrained transport
    problem have the same infimum. In other words, their risks can be made
    equally small even if no optimizer attains that limit. If the transport
    problem attains its infimum and suitable optimal deterministic transport
    maps, called Monge maps, exist, then a predictor built from those maps
    realizes the common value. Under a standard one-dimensional condition that
    makes increasing transport optimal---the Spence-Mirrlees or twist
    condition---the relevant maps are nondecreasing. This leads to an
    optimization over bounded group-specific quantile functions. The
    optimization is convex whenever the conditional expected loss, which we
    call the induced cost, is convex in the prediction.

    \item \textbf{Squared-loss geometry.}
    Under the squared-loss model stated below and with two groups that each
    have positive probability, we derive in closed form the optimal adjusted
    group quantile functions for the $W_1$, $W_2$, and $W_\infty$
    constraints. These functions determine the corresponding approximately
    fair predictor: for each observation, the predictor finds its rank within
    its group and returns the adjusted quantile at that rank. To derive the
    formulas, we modify the difference between the two group quantile
    functions while preserving their weighted average, and then recover both
    adjusted quantile functions. The $W_1$
    constraint uses soft-thresholding: it sets gaps whose
    magnitude is below a common threshold to zero and reduces larger gaps by
    that threshold. The $W_2$ constraint shrinks the entire quantile-gap
    function by a common factor when the constraint is active. The $W_\infty$
    constraint clips each gap so that its magnitude does not exceed
    $\epsilon$. These formulas make the effect of each constraint directly
    visible.

    \item \textbf{Estimation and computation.}
    We construct a known-cost plug-in estimator and
    establish its convergence rates under the stated regularity conditions.
    For computation, we develop a finite-grid projected-gradient algorithm.
    At each outer iteration, a gradient step is followed by a group-weighted
    Dykstra projection that jointly enforces pairwise fairness, nondecreasing
    quantile functions, and prediction bounds. For unknown induced costs, we
    also provide an observed-loss computational proxy based on observed
    responses.
    
\end{itemize}

\paragraph{Related work.}
The closest work studies regression under strict DP. For squared loss,
\citet{chzhen2020fair} show that the optimal fair prediction distribution is a
Wasserstein barycenter. \citet{feng2026fair} develop a strict-DP framework for
more general loss functions. Both settings require every group to have the
same prediction distribution. Our setting instead allows group-specific
distributions, provided that every pair remains within the chosen Wasserstein
tolerance.

Other papers relax fairness in different ways. \citet{agarwal2019fair} give
reduction-based algorithms for two criteria. One limits differences between
each group's prediction distribution function and the overall distribution
function. The other limits prediction loss within each group.
\citet{xu2023pareto} study a fairness-utility frontier based on average
pairwise Wasserstein disparity. We instead place a hard bound on the largest
pairwise distance. \citet{chzhen2022minimax} measure improvement relative to
the unfairness of the unconstrained regression function. Our tolerance is an
absolute Wasserstein bound, and it is not scaled by that baseline unfairness.

Optimal transport has also been used to repair data and predictions.
\citet{gordaliza2019obtaining} move group-conditional data distributions
toward a Wasserstein barycenter for fair classification.
\citet{jiang2020wasserstein} use $W_1$ distances to enforce independence
between classifier outputs and sensitive information.
\citet{legouic2020projection} post-process a regression estimator into the
closest estimator satisfying strict DP. More recent work uses a demographic
parity penalty rather than a hard constraint \citep{lince2026geometry}, or
enforces parity only at selected quantiles or score thresholds
\citep{charpentier2026localized}. In contrast, we impose one absolute
tolerance on the maximum Wasserstein distance over all pairs of group
prediction distributions. Our transport arguments use standard tools from
one-dimensional optimal transport \citep{santambrogio2015optimal}, and the
$W_1$ numerical projection uses efficient projection onto an $\ell_1$ ball
\citep{duchi2008efficient}.

The remainder of this paper is organized as follows. Section \ref{sec:setup}
introduces the population problem and its optimal-transport formulation. It
then gives the squared-loss closed forms and the extensions to convex induced
costs and multiple groups. Section \ref{sec:estimation} presents the
known-cost plug-in estimator and the numerical procedures. Section
\ref{sec:theory} gives consistency and convergence rates for this estimator.
Section \ref{sec:numerical} presents regression
experiments and a classification illustration. Section \ref{sec:real_data}
applies the method to income prediction with California ACS data. Section
\ref{sec:conclusion} concludes the paper.

\section{Problem Setup and Population Characterization} \label{sec:setup}

Consider a regression problem.
Let
$X\in\mathcal X\subseteq\mathbb R^d$ contain the predictors, $Y\in\mathbb R$
be the response, and $S\in\mathcal S:=\{1,\ldots,K\}$ indicate the sensitive group. Write
$r_s=\mathbb P(S=s)$ and assume $r_s>0$.
We suppose that predictions lie in a compact
interval $\mathcal I=[a,b]$.
Their accuracy is measured by a loss
$\ell(y,q)$, where the prediction $q$ lies in $\mathcal I$.

%

We assess fairness by comparing prediction distributions across groups.
For real-valued predictions, each distribution is characterized by its
quantile function, which maps a probability level to the corresponding
prediction value. Specifically, for a distribution function $F$, the
quantile function is defined by
\[
Q(u):=\inf\{t:F(t)\ge u\},\qquad u\in(0,1).
\]
We write
\[
\mathcal Q_{\mathcal I}
:=
\{Q:(0,1)\to\mathcal I:
Q\text{ is nondecreasing and left-continuous}\}
\]
for the class of bounded quantile functions considered throughout the
paper. We extend each $Q\in\mathcal Q_{\mathcal I}$ to $[0,1]$ by setting
\[
Q(0):=\lim_{u\downarrow0}Q(u),\qquad
Q(1):=\lim_{u\uparrow1}Q(u).
\]
These limits exist because $Q$ is bounded and nondecreasing. Unless stated
otherwise, an equality or constraint between quantile functions need only
hold outside a set of ranks with total length zero. Equivalently, it holds
almost everywhere with respect to Lebesgue measure.

We next introduce the transport notation used to compare group distributions.
A coupling is a
way to pair draws from two
distributions $\mu$ and $\nu$. Formally, it is a probability measure $\gamma$
on $\mathbb R^2$. Its first coordinate has distribution $\mu$, and its second
coordinate has distribution $\nu$; these coordinate distributions are called
the marginals. We write $\Pi(\mu,\nu)$ for all such pairings. If $Z$ has
distribution $\mu$ and we apply a measurable map $T$, the distribution of
$T(Z)$ is the push-forward of $\mu$. For any measurable set $B$ of possible
output values, $(T_\#\mu)(B)=\mu(T^{-1}(B))$, where $T^{-1}(B)$ is the set of
inputs that $T$ maps into $B$.

\begin{definition}[Demographic parity]
A predictor $f:\mathcal X\times\mathcal S\to\mathcal I$ satisfies
demographic parity if the conditional law of $f(X,S)$ given $S=s$ is the same
for every $s\in\mathcal S$.
\end{definition}

In other words, the prediction and the group label are independent. We write
this as $f(X,S)\perp S$.

We first define the population target without restricting the form of the
predictor. Let $\mathcal F_{\mathcal I}$ contain all measurable predictors
from $\mathcal X\times\mathcal S$ to $\mathcal I$. This broad class includes
every transport-based adjustment used below. A smaller model class could lead to additional
error because it may not contain the best adjustment. Under strict
demographic parity, the target fair regression function is
\begin{equation}\label{FP}
\underset{f\in\mathcal F_{\mathcal I}:\ f(X,S)\perp S}{\arg\inf}
\mathbb E\{\ell(Y,f(X,S))\}.
\end{equation}

\subsection{Connection to the Kantorovich Barycenter Problem}

The transport construction begins by reducing $(X,S)$ to a one-dimensional score. For this reduction to preserve the relevant distributional information, the score must be sufficient for the conditional distribution of $Y$. We formalize this requirement as follows.

\begin{assumption}\label{assum:f_star}
There exists a measurable $f^*:\mathcal X\times\mathcal S\to\mathcal I$ such
that $Y\perp(X,S)\mid f^*(X,S)$.
\end{assumption}

{Assumption~\ref{assum:f_star} requires the scalar score $f^*(X,S)$ to retain all information in $(X,S)$ relevant to the conditional distribution of $Y$. This is stronger
than ordinary regression, which focuses only on the conditional mean. For
squared loss, however, it is enough to use
\(f^*(X,S)=\mathbb E(Y\mid X,S)\). The remaining noise adds the same amount to the
risk of every predictor and therefore does not change which predictor is the best.}

For a latent score \(z\) and prediction \(q\), the induced cost is the
conditional expected loss
\[
\mathcal C(z,q):=\mathbb E\{\ell(Y,q)\mid f^*(X,S)=z\}.
\]
Taking the conditional expectation first gives
\begin{align}\label{eq-cost}
\mathbb E\{\ell(Y,f(X,S))\}
=\mathbb E\{\mathcal C(f^*(X,S),f(X,S))\}.
\end{align}
The prediction risk can therefore be viewed as the cost of changing a latent
score into a final prediction. For group $s$, let
$\mu_s=\mathcal L(f^*(X,S)\mid S=s)$ denote the distribution of its latent
scores. Let $F_s^*$ and $Q_s^*$ be the corresponding distribution and
quantile functions.

Non-atomic means that no single latent-score value has positive probability.
\begin{assumption}\label{assum:cost_function}
Each $\mu_s$ is non-atomic, $\mathcal C$ is continuous on
$\mathcal I^2$, and every deterministic transport statement below is made
only when the corresponding Monge problem admits an optimal measurable map.
\end{assumption}

The next distinction will matter in Theorem~\ref{thm1}. Pairing probability
mass through a coupling may allow randomization. A Monge map is instead a
deterministic rule that sends each latent score to one adjusted value. Such a
map is needed only when we want to construct a deterministic predictor.

We also want higher latent scores to be paired with higher adjusted values.
The following standard one-dimensional condition gives this ordering.

\begin{assumption}\label{assum:cost_function2}
The cost is twice differentiable on $\mathcal I^2$ and satisfies the
Spence-Mirrlees condition:
$\partial^2_{zq}\mathcal C(z,q)<0$ \citep{villani2009optimal,chiappori2017multi}.
\end{assumption}
Under this condition, sorting both variables in the same order gives an
optimal pairing. This operation is called increasing rearrangement, and the
resulting transport maps are nondecreasing. Under strict demographic parity,
all groups must share one target distribution. This recovers the barycenter
formulation of \citet{feng2026fair}: the barycenter is the shared distribution
that best balances the group-specific losses.

\subsection{Relaxation to Approximate Fairness}

Approximate parity allows the group prediction distributions to differ, but
only by a chosen tolerance. Let $\nu_{f,s}$ be the prediction distribution
for group $s$. We measure the difference between two such distributions by
the order-$p$ Wasserstein distance $W_p$, which is defined formally in
Section~\ref{sec:wasserstein_constraint}. For $p\in[1,\infty]$ and
$\epsilon\ge0$, the approximate fair-regression problem is
\begin{equation}\label{CFP_cost}
\begin{aligned}
\inf_{f\in\mathcal F_{\mathcal I}}\quad&
\mathbb E\{\mathcal C(f^*(X,S),f(X,S))\}\\
\text{s.t.}\quad&W_p(\nu_{f,i},\nu_{f,j})\le\epsilon,
\qquad 1\le i<j\le K.
\end{aligned}
\end{equation}
Setting $\epsilon=0$ gives strict demographic parity.

We can state the same task directly in terms of distributions. For each
group, choose an adjusted prediction distribution and a way to pair its
latent scores with predictions from that distribution. The resulting
distributional problem is called the constrained Kantorovich problem. It
extends the corresponding Kantorovich barycenter problem: when $\epsilon=0$,
every group must choose the same distribution. For $\epsilon>0$, the group
distributions may differ, so the solution need not be a single barycenter:
\begin{equation}\label{KP}
\begin{aligned}
\inf_{\nu_1,\ldots,\nu_K\in\mathcal P(\mathcal I)}\quad&
\sum_{s=1}^K r_s\inf_{\gamma_s\in\Pi(\mu_s,\nu_s)}
\int\mathcal C(z,q)\,d\gamma_s(z,q)\\
\text{s.t.}\quad&W_p(\nu_i,\nu_j)\le\epsilon,
\qquad 1\le i<j\le K.
\end{aligned}
\end{equation}
Here $\mathcal P(\mathcal I)$ contains all probability distributions supported
on the allowed prediction range $\mathcal I$.

Theorem~\ref{thm1} answers two different questions. First, do the prediction
problem and the distributional problem have the same best achievable value?
Second, when can that value be achieved by a deterministic prediction rule?
The first conclusion needs only the base assumptions. The second also needs
the optimum and the required deterministic transport maps to exist.
For brevity, we call the approximate fair-regression
problem AFRP and the constrained Kantorovich problem CKP.

\begin{theorem}[Equality of infima and conditional attainment]\label{thm1}
Under Assumptions~\ref{assum:f_star} and \ref{assum:cost_function}, the AFRP
and CKP infima in \eqref{CFP_cost} and \eqref{KP} are equal. If the CKP
attains its infimum at $(\bar\nu_1,\ldots,\bar\nu_K)$ and, for every $s$, an
optimal Monge map $T_s$ from $\mu_s$ to $\bar\nu_s$ exists, then
\[
f_{\mathrm{fair}}^*(x,s)=T_s(f^*(x,s))
\]
attains the AFRP infimum. Without CKP or Monge attainment the theorem asserts
only equality of infima.
\end{theorem}

Thus the distributional problem always identifies the best value of the
prediction problem under the stated assumptions. To obtain the displayed
deterministic predictor, the distributional optimum must be attained and each
group must have the required Monge map.

\subsection{Characterization via Quantile Functions}
On the real line, the transport rule has a simple interpretation: first find
the within-group rank of a latent score, and then return the adjusted value at the
same rank. When the optimal deterministic transport map exists,
Assumption~\ref{assum:cost_function2} makes it nondecreasing, with the form
$T_s=Q_s\circ F_s^*$. The predictor is therefore
\begin{equation}\label{eq-opt-form-approx}
f(x,s)=Q_s\{F_s^*(f^*(x,s))\}.
\end{equation}
The within-group rank is uniformly distributed:
$F_s^*(f^*(X,s))\mid(S=s)\sim\operatorname{Unif}(0,1)$. It follows that
$Q_s$ determines the final prediction distribution for group $s$.

\subsection{Wasserstein Distance Constraint}\label{sec:wasserstein_constraint}
Wasserstein distance compares two distributions by pairing their values and
measuring how far the paired values are apart. For $p\in[1,\infty)$, define
$W_p(\nu_1,\nu_2):=(\inf_{\gamma\in\Pi(\nu_1,\nu_2)}
\mathbb{E}_{(X,Y)\sim\gamma}|X-Y|^p)^{1/p}$, and define
\[
W_\infty(\nu_1,\nu_2)
:=\inf_{\gamma\in\Pi(\nu_1,\nu_2)}
\operatorname*{ess\,sup}_{(X,Y)\sim\gamma}|X-Y|.
\]
The finite-$p$ Wasserstein distance uses an average of the (powered) paired differences.
$W_\infty$ instead records the largest difference, apart from events with
probability zero. {On the real line, the best pairing compares
values at the same percentile in the two distributions. We can therefore
compute both distances from their quantile functions.} To state the following lemma, we introduce the norm notation. For a measurable function $h:(0,1)\to\mathbb{R}$, write
$\|h\|_p=(\int_0^1 |h(u)|^p\,du)^{1/p}$ for $1\le p<\infty$
and $\|h\|_\infty=\operatorname*{ess\,sup}_{u\in(0,1)}|h(u)|$,
with respect to Lebesgue measure.

\begin{lemma}\label{Wsup lemma}
For probability measures $\nu_1,\nu_2$ on $\mathcal I$ with quantiles
$Q_{\nu_1},Q_{\nu_2}$,
\[
W_p(\nu_1,\nu_2)=\|Q_{\nu_1}-Q_{\nu_2}\|_p,
\quad 1\le p<\infty,
\qquad
W_\infty(\nu_1,\nu_2)=\|Q_{\nu_1}-Q_{\nu_2}\|_\infty.
\]
\end{lemma}
{This standard formula lets us express fairness directly as a
bound on the distance between quantile curves}
\citep[Proposition~2.17]{santambrogio2015optimal}.

We study $p\in\{1,2,\infty\}$. The $p=1$ distance measures the average
absolute gap across ranks. The $p=2$ distance measures the root-mean-square
gap. The $p=\infty$ distance bounds the gap at every rank.

Define the bounded feasible set
\[
\mathcal K_{p,\epsilon,\mathcal I}:=
\left\{(Q_1,\ldots,Q_K)\in\mathcal Q_{\mathcal I}^K:
\|Q_i-Q_j\|_p\le\epsilon\ \text{for every }i<j\right\}.
\]
This set imposes three requirements at once. Each curve must be a valid
quantile function, every prediction must lie in $\mathcal I$, and every pair
of curves must satisfy the fairness tolerance. {We can therefore
solve the distributional problem by choosing one adjusted quantile function
for each group:}
\begin{equation}\label{Final-Opt}
\inf_{\mathbf Q\in\mathcal K_{p,\epsilon,\mathcal I}}
G(\mathbf Q):=
\sum_{s=1}^K r_s\int_0^1
\mathcal C(Q_s^*(u),Q_s(u))\,du.
\end{equation}
{Assumption~\ref{assum:cost_function2} says, roughly, that lower
scores should be paired with lower predictions and higher scores with higher
predictions. This way of pairing them is the least costly. Therefore, the
problem over group distributions and \eqref{Final-Opt} have the same lowest
possible value.}
This is a convex optimization problem whenever
$q\mapsto\mathcal C(z,q)$ is convex. All integrals are finite because the
prediction range is bounded. The common index $u$ compares the same
within-group rank across groups. At rank $u$, $Q_s^*(u)$ is group $s$'s
latent-score value, and $Q_s(u)$ is the corresponding adjusted prediction
value.

Formally, for a
measurable $h:(0,1)\to\mathcal I$, let $h^\uparrow$ be the quantile function
of $h(U)$ when $U\sim\operatorname{Unif}(0,1)$.

\begin{lemma}[Simultaneous monotone rearrangement]\label{lem:rearrangement}
Let $h_1,\ldots,h_K$ be measurable, integrable, and $\mathcal I$-valued.
For every $p\in[1,\infty]$,
\[
\|h_i^\uparrow-h_j^\uparrow\|_p\le\|h_i-h_j\|_p.
\]
If $Q_s^*$ is nondecreasing and $\mathcal C$ satisfies the
Spence-Mirrlees condition, then
\[
\int_0^1\mathcal C(Q_s^*(u),h_s^\uparrow(u))\,du
\le
\int_0^1\mathcal C(Q_s^*(u),h_s(u))\,du.
\]
Moreover, $h_s^\uparrow$ is nondecreasing, left-continuous, integrable, and
has range in $\mathcal I$.
\end{lemma}
The lemma says that sorting cannot make two group curves farther apart. Under
the stated cost condition, it also cannot increase prediction loss. The
sorted curves remain bounded, integrable quantile functions. We use these
facts to construct a valid optimizer for the pointwise $W_\infty$ problem and
to prove the multi-group squared-loss barycenter property in
Proposition~\ref{prop:barycenter_K}.

\subsection{Closed-Form Solutions for Squared Loss}\label{sec:closed_form}

The general problem \eqref{Final-Opt} usually requires numerical
optimization. With squared loss and two groups ($K=2$), the adjustment has a
simple geometry and a closed-form solution. Appendix~\ref{app:derivations}
gives the detailed derivations.

In this setting, we assume an additive error model
$Y=f^*(X,S)+\varepsilon$, where the noise $\varepsilon$ is independent of
$(X,S)$, has mean zero, and has finite variance
$\sigma^2=\operatorname{Var}(\varepsilon)$.
Under this model, $\mathcal C(z,q)=(z-q)^2+\sigma^2$. The variance term is
constant with respect to $q$. The adjustment therefore keeps the final
quantile curves as close as possible to the latent curves under squared
$L_2$ distance. Assume $r_1,r_2>0$ and $r_1+r_2=1$. The fair quantiles solve
\begin{equation} \label{TwoGroupProb}
\begin{aligned}
    &\min_{Q_1, Q_2\in\mathcal Q_{\mathcal I}}
    \left( r_1 \|Q_1^* - Q_1\|_2^2 + r_2 \|Q_2^* - Q_2\|_2^2 \right) \\
    &\text{s.t.} \quad \|Q_1-Q_2\|_p \le \epsilon.
\end{aligned}
\end{equation}
By Lemma~\ref{Wsup lemma}, the Wasserstein distance between the two adjusted
distributions, with quantiles $Q_1$ and $Q_2$, is $\|Q_1-Q_2\|_p$. We examine
$p\in\{1,2,\infty\}$. We first solve for the difference between the curves.
We then verify that the resulting curves are still bounded, nondecreasing
quantile functions.

\subsubsection{Solution Structure}
The fairness constraint acts on the difference between the two groups. It does
not directly restrict their weighted average. For candidate quantile functions
$Q_1$ and $Q_2$, define
\[
\bar Q=r_1Q_1+r_2Q_2,\qquad
\Delta=Q_1-Q_2,
\]
and let $\bar Q^*=r_1Q_1^*+r_2Q_2^*$ and
$\Delta^*=Q_1^*-Q_2^*$. A direct expansion gives
\[
r_1\|Q_1^* - Q_1\|_2^2
+
r_2\|Q_2^* - Q_2\|_2^2
=
\|\bar Q^*-\bar Q\|_2^2
+
r_1 r_2 \|\Delta^* - \Delta\|_2^2 .
\]
The identity separates the loss into a term for the weighted average and a
term for the group difference. Changing the weighted average would increase
loss without helping the fairness constraint. The unique squared-loss
solution therefore keeps $\bar Q=\bar Q^*$.
Next, we choose $\Delta=Q_1-Q_2$ as close as possible to
$\Delta^*=Q_1^*-Q_2^*$ while requiring $\|\Delta\|_p\le\epsilon$. Once the
new gap is known, the two curves are
$$ Q_1 = \bar{Q}^* + r_2 \Delta, \quad Q_2 = \bar{Q}^* - r_1 \Delta. $$

\subsubsection{Case 1: \texorpdfstring{$L_2$}{L2}-Constraint
(\texorpdfstring{$p=2$}{p=2})}
The $L_2$ rule changes the whole gap curve by one common factor. If the
original gap is too large, every value of that curve is multiplied by the same
number below one. Geometrically, this is a radial projection onto an $L_2$
ball.
Let $D = \|\Delta^*\|_2$.
\begin{itemize}
    \item If $D \le \epsilon$, the constraint is inactive, and $Q_s = Q_s^*$ (the latent predictors are already sufficiently fair).
    \item If $D > \epsilon$, the difference is linearly shrunk towards zero. The optimal difference is $\Delta = \frac{\epsilon}{D} \Delta^*$.
\end{itemize}
Thus the solution shrinks the two latent curves toward their strict-parity
barycenter while preserving their weighted average.

\subsubsection{Case 2: \texorpdfstring{$L_1$}{L1}-Constraint
(\texorpdfstring{$p=1$}{p=1})}
The $L_1$ rule subtracts the same threshold from the size of the gap at every
rank. Gaps below that threshold become zero, while larger gaps keep their
sign. This operation is called soft-thresholding. It solves the projection of
$\Delta^*$ onto the $L_1$ ball $\|\Delta\|_1\le\epsilon$ under squared
$L_2$ loss.

For a threshold $\lambda\ge0$, apply the rule to
$\Delta^*(u)=Q_1^*(u)-Q_2^*(u)$:
\begin{align*}
    \Delta(u) = \operatorname{sgn}(\Delta^*(u)) \cdot \max(0, |\Delta^*(u)| - \lambda).
\end{align*}
Choose $\lambda$ to satisfy the fairness constraint. If
$\|\Delta^*\|_1\le\epsilon$, the constraint is inactive, so $\lambda=0$ and
$\Delta=\Delta^*$. If $0<\epsilon<\|\Delta^*\|_1$, there is a unique
$\lambda>0$ such that $\int_0^1|\Delta(u)|du=\epsilon$. At exact parity,
$\epsilon=0$, the solution is $\Delta=0$; the threshold itself need not be
unique once it exceeds $\|\Delta^*\|_\infty$.

\subsubsection{Case 3: \texorpdfstring{$L_\infty$}{L-infinity}-Constraint
(\texorpdfstring{$p=\infty$}{p=infinity})}
The $L_\infty$ rule leaves a gap unchanged when its size is at most
$\epsilon$. It caps every larger positive or negative gap at
$\epsilon$ in absolute value, which corresponds to clipping. Formally, the constraint is
$|Q_1(u)-Q_2(u)|\le\epsilon$ for almost all $u\in[0,1]$, so each
$\Delta^*(u)$ is projected onto $[-\epsilon,\epsilon]$:
\begin{align*}
    \Delta(u) = \operatorname{sgn}(\Delta^*(u)) \cdot \min(\epsilon, |\Delta^*(u)|) = 
    \begin{cases} 
    \epsilon & \text{if } \Delta^*(u) > \epsilon, \\
    \Delta^*(u) & \text{if } |\Delta^*(u)| \le \epsilon, \\
    -\epsilon & \text{if } \Delta^*(u) < -\epsilon.
    \end{cases}
\end{align*}
At each rank, the allowed pairs form the band
$\{(x,y)\in\mathcal I^2:|x-y|\le\epsilon\}$. The rule moves an infeasible
pair to the edge of this band while keeping its weighted average fixed.

These calculations solve for a gap function, but a valid solution also needs
two genuine quantile functions. The next lemma shows that the recovered curves
are nondecreasing, left-continuous, integrable, and confined to $\mathcal I$.
\begin{lemma}[Monotonicity of the two-group closed-form solutions]
\label{lem:monotonicity_closed_form}
Assume that $Q_1^*,Q_2^*\in\mathcal Q_{\mathcal I}$ and that $r_1,r_2>0$
with $r_1+r_2=1$. Let
\[
    \bar Q^* = r_1 Q_1^*+r_2 Q_2^*, \qquad
    \Delta^*=Q_1^*-Q_2^*.
\]
Suppose that the adjusted difference has the form
\[
    \Delta(u)=\psi(\Delta^*(u)),
\]
where $\psi:\mathbb R\to\mathbb R$ is non-decreasing, $1$-Lipschitz, and
$\psi(0)=0$. Define
\[
    Q_1=\bar Q^*+r_2\Delta, \qquad Q_2=\bar Q^*-r_1\Delta .
\]
Then $Q_1,Q_2\in\mathcal Q_{\mathcal I}$. Consequently the three
closed-form solutions preserve monotonicity, left continuity, integrability,
and the bounded range.
\end{lemma}

\begin{table}[t]
\centering
\small
\setlength{\tabcolsep}{7pt}
\renewcommand{\arraystretch}{1.2}
\caption{\textbf{Closed-form $\Delta$ updates (squared loss, $K=2$).}
Let $\Delta^*=Q_1^*-Q_2^*$ and $\Delta=Q_1-Q_2$, with constraint $\|\Delta\|_p\le\epsilon$.
Write $D=\|\Delta^*\|_2$, $\mathcal S_\lambda(t)=\mathrm{sgn}(t)(|t|-\lambda)_+$,
and $\mathrm{clip}_{[-\epsilon,\epsilon]}(t)=\min\{\epsilon,\max\{-\epsilon,t\}\}$.
For $p=1$, choose $\lambda$ so that
$\|\mathcal S_\lambda(\Delta^*)\|_1=\epsilon$ when
$0<\epsilon<\|\Delta^*\|_1$. Use $\lambda=0$ when the constraint is inactive.
Use $\Delta=0$ when $\epsilon=0$.}
\label{tab:closed_form_summary_compact2}

\begin{tabularx}{0.90\linewidth}{@{}%
  >{\centering\arraybackslash}p{0.07\linewidth}%
  X%
  >{\centering\arraybackslash}p{0.16\linewidth}@{}}
\toprule
$p$ & Optimal difference $\Delta$ & Effect \\
\midrule
$2$ & $\Delta=\Delta^*$ if $D\le\epsilon$;\quad otherwise $\Delta=(\epsilon/D)\Delta^*$. & shrink \\
$1$ & $\Delta=\mathcal S_\lambda(\Delta^*)$. & soft-thresh \\
$\infty$ & $\Delta=\mathrm{clip}_{[-\epsilon,\epsilon]}(\Delta^*)$. & clip \\
\bottomrule
\end{tabularx}
\end{table}

\subsection{Extension to General Convex Losses and Multiple Groups} \label{sec:extension}

The closed forms use both squared loss and the special geometry of two groups.
{The quantile problem above also applies to more than two groups,
provided the same assumptions hold. In particular,
Assumption~\ref{assum:cost_function2} is needed when we match scores and
predictions by rank. If the conditional expected loss is convex in the
prediction, the resulting problem is convex.}
The computation depends on how the fairness tolerance is distributed across
ranks. Under $W_\infty$, every rank has its own bound. Under finite-$p$
distances, all ranks share one overall budget.

\paragraph{Pointwise Wasserstein constraint.}
When $p=\infty$, every pair of adjusted values must be close at each rank:
\[
    |Q_i(u)-Q_j(u)|\le \epsilon
    \quad\text{for almost every }u\in(0,1),\quad i,j\in\mathcal S.
\]
We can first solve this problem separately at each rank and restore the
quantile ordering afterward. At a fixed rank $u$, let $\theta_s(u)$ be a
prediction that minimizes the loss for group $s$ without a fairness
constraint, with these values chosen measurably as $u$ varies. When
$0\le\epsilon<b-a$, fairness requires all adjusted group values to fit in one
interval of width $\epsilon$. There is a pointwise optimum of the form
\begin{equation} \label{eq:pointwise_convex}
q_s(u;c)=\operatorname{clip}_{[c,c+\epsilon]}\{\theta_s(u)\},
\qquad c\in[a,b-\epsilon],
\end{equation}
where $c=c^*(u)$ minimizes
$\sum_s r_s\mathcal C(Q_s^*(u),q_s(u;c))$. In words, choose the left endpoint
of a common interval and clip every group's preferred value into that
interval. Only the single endpoint $c$ must be optimized. If
$\epsilon\ge b-a$, every value in $\mathcal I$ already fits within the
tolerance.

When $0\le\epsilon<b-a$, the endpoint may change with the rank $u$. We need
to choose these endpoints so that they form a measurable function. The
possible endpoints lie in the compact interval $[a,b-\epsilon]$. The objective
is measurable in $u$ and continuous in $c$. A standard selection result then
guarantees that the endpoints can be chosen measurably. This result is the
measurable maximum theorem
\citep[Theorem~18.19]{aliprantisborder2006}.

After clipping, the resulting curves $q_s(u;c^*(u))$ may be out of order as
$u$ increases. We sort each curve into quantile order by increasing
rearrangement. {Sorting does not move any prediction outside the
allowed range or make any pair of curves farther apart.}
Thus the construction has two stages.
First, solve for one endpoint and clip at each rank. After all ranks have been
processed, sort each completed curve into quantile order. This is an explicit
clipping rule with a one-dimensional endpoint solve, although $c^*(u)$ need
not have a closed-form formula for general $K$.


Under squared loss, consider the class of collections satisfying the pairwise fairness constraints. The projection of the latent quantile curves onto this class under the group-weighted $L_2$ distance is unique. Moreover, this projection preserves the weighted barycenter even when $K>2$; see Proposition~\ref{prop:barycenter_K}.
A convenient stronger condition also ensures the pairwise tolerance. Specifically, require each group to lie within $\epsilon/2$ of the fixed latent weighted barycenter
$\bar Q^*=\sum_s r_s Q_s^*$
under the same Wasserstein distance. By the triangle inequality, any two groups are then at most
$\epsilon/2+\epsilon/2=\epsilon$
apart. Thus the radius $\epsilon/2$ guarantees the desired pairwise constraint. The factor $1/2$ is necessary for this argument: using radius $\epsilon$ would guarantee only a pairwise distance of at most $2\epsilon$.

\paragraph{Integral Wasserstein constraints.}
When $1\le p<\infty$, the constraint
\[
    \|Q_i-Q_j\|_p\le\epsilon
\]
uses one total budget across all ranks. We therefore cannot optimize each rank
separately. For computation, we represent each curve on a midpoint grid:
\[
u_k=(k-1/2)/M,\qquad
\mathbf q_s=(q_{s,1},\ldots,q_{s,M})^\top.
\]
The grid problem is
\begin{equation} \label{eq:discretized_prob}
\min_{\mathbf q}\quad
\frac1M\sum_{s=1}^K r_s\sum_{k=1}^M
\mathcal C(q_{s,k}^*,q_{s,k})
\quad\text{s.t.}\quad
\mathbf q\in\mathcal K_{p,M}\cap\mathcal M^K\cap\mathcal B_{\mathcal I},
\end{equation}
where
\[
\mathcal K_{p,M}=\left\{\mathbf v:
\left(M^{-1}\sum_{k=1}^M|v_{i,k}-v_{j,k}|^p\right)^{1/p}
\le\epsilon\ \forall i<j\right\}
\]
for finite $p$, with the maximum norm for $p=\infty$;
\[
\mathcal M=\{\mathbf q:q_1\le\cdots\le q_M\},\qquad
\mathcal B_{\mathcal I}=\{\mathbf q:a\le q_{s,k}\le b\ \forall s,k\}.
\]
The three sets have clear roles. The first enforces pairwise fairness. The
second keeps every group curve nondecreasing. The third keeps every prediction
inside $\mathcal I$. Thus the prediction bounds are enforced in the population
problem, the grid problem, and each numerical projection.

\section{Estimation Procedure} \label{sec:estimation}

The population problem gives one adjusted quantile function for each group.
This section explains how to estimate those functions and use them for new
predictions.

The data procedure has four steps. First, train an unconstrained predictor.
Second, use a separate calibration sample to estimate the distribution of its
scores within each group. Third, adjust the estimated group quantiles to meet
the fairness tolerance. Finally, map each new score to its within-group rank
and read the prediction from the adjusted quantile curve.

There are several ways to perform the adjustment. With squared loss and two
groups, a closed-form rule gives the full adjusted curves. More generally,
when the induced cost can be calculated, the full quantile-function problem
defines the known-cost plug-in estimator studied in
Section~\ref{sec:theory}. Under the
regularity assumptions stated there, we establish consistency and convergence
rates for this estimator. A numerical implementation represents the curves on
a finite midpoint grid. Each iteration first takes a step in the
negative-gradient direction. It then projects the grid values back onto the
fairness, ordering, and prediction-range constraints. This is
projected-gradient descent. The projection is computed by Dykstra's method,
which repeatedly projects onto the individual constraints while correcting
for earlier projections. When the induced cost cannot be calculated, an
observed-loss version instead uses the response paired with each ordered
fitted score.

\subsection{Stage 1: Unconstrained group-wise regression}
First, we estimate $f^*$ using an auxiliary training sample
$\mathcal D_m^{\mathrm{tr}}$ of size $m$, and let
\[
    \mathcal D_n^{\mathrm{cal}}=\{(X_i,S_i,Y_i):i=1,\ldots,n\}
\]
denote an independent calibration sample of size $n$. A group-aware regression
may fit separate models or may include group indicators and interactions in
one model. The theory in Section \ref{sec:theory} requires the training and calibration samples to be
independent. We call $\hat Y_i=\hat f(X_i,S_i)$ a \emph{fitted calibration
score}.
Specifically, the two steps in this stage are listed as follows.
\begin{enumerate}
    \item Train group-wise regression models $\hat f(\cdot,s)$ for $s=1,\ldots,K$.
    \item Apply the trained models to the calibration sample \(\mathcal D_n^{\mathrm{cal}}\) and compute
    \[
        \hat Y_i=\hat f(X_i,S_i),\qquad i=1,\ldots,n.
    \]
\end{enumerate}

\subsection{\texorpdfstring{\mbox{}}{}Stage 2: Empirical latent quantile estimation}
Let \(n_s=\sum_{i=1}^n\mathbbm 1\{S_i=s\}\) be the number of calibration
observations in group $s$, and let \(\hat r_s=n_s/n\) be its sample
proportion. Within each group, sort the fitted scores:
\[
    \hat Y^{(s)}_{(1)}\le \hat Y^{(s)}_{(2)}\le\cdots\le \hat Y^{(s)}_{(n_s)}.
\]
The empirical latent quantile function is
\[
    \hat Q_s^*(u)=\hat Y^{(s)}_{(\lceil n_su\rceil)},\qquad u\in(0,1].
\]
At rank $u$, $\hat Q_s^*$ returns the fitted score in position
$\lceil n_su\rceil$. Thus it is the exact inverse empirical distribution
function of the fitted scores. At a midpoint $u_k$, we use that order statistic
directly. We do not interpolate sample quantiles at this stage. Interpolation
is used only after the fairness adjustment has been computed.

The estimator is formed when every calibration group is present, that is, on
the event $\min_s n_s>0$. The positive-proportion assumption below makes this
event increasingly likely as the sample grows. The observed-loss method in
Remark~\ref{rem:unknown_cost} also keeps each score paired with its own
response while sorting. The response paired with the $k$th ordered score
$\hat Y^{(s)}_{(k)}$ is called its \emph{concomitant} and is written
$Y^{(s)}_{(k)}$. Only the fitted scores are put in increasing order; the
responses are carried along rather than sorted by their values.

\subsection{Stage 3: Fairness adjustment}
The third stage turns the empirical latent curves $\hat Q_s^*$ into adjusted
curves $\hat Q_s$ that satisfy the fairness tolerance. Squared loss with two
groups uses a closed-form plug-in rule. For a known induced cost, the
full quantile-function problem defines the estimator, and a
projected-gradient calculation provides a grid approximation. If the induced
cost is unavailable, the observed-loss calculation changes only the grid
objective. It uses the same fairness, ordering, and prediction-range
constraints.

\subsubsection{Closed-form adjustment for squared loss and two groups}
\label{sec:estimation_closed_form}
For squared error loss with $K=2$, the formulas in
Section~\ref{sec:closed_form} give a direct plug-in estimator. First, define
the empirical weighted average and the difference between the two group
curves:
\[
    \hat{\bar Q}(u)=\hat r_1\hat Q_1^*(u)+\hat r_2\hat Q_2^*(u),
    \qquad
    \hat\Delta^*(u)=\hat Q_1^*(u)-\hat Q_2^*(u).
\]
The fairness adjustment applies the appropriate projection to
$\hat\Delta^*$ and returns a new difference $\hat\Delta$. The rules for the
three Wasserstein distances are summarized in Table~\ref{tab:estimation_closed_form}.

\begin{table}[t]
\centering
\small
\setlength{\tabcolsep}{6pt}
\renewcommand{\arraystretch}{1.25}
\caption{\textbf{Plug-in fairness adjustment for squared loss and two groups.}
Let $\hat\Delta^*=\hat Q_1^*-\hat Q_2^*$ and let
$\hat\Delta=\hat Q_1-\hat Q_2$. For $W_2$, write
$\hat D=\|\hat\Delta^*\|_2$. For $W_1$, choose $\hat\lambda$ so that
$\|\mathcal S_{\hat\lambda}(\hat\Delta^*)\|_1=\epsilon$ when
$0<\epsilon<\|\hat\Delta^*\|_1$. Set $\hat\lambda=0$ when the constraint is
inactive. Use $\hat\Delta=0$ when $\epsilon=0$.}
\label{tab:estimation_closed_form}
\begin{tabularx}{0.95\linewidth}{@{}p{0.16\linewidth}X p{0.20\linewidth}@{}}
\toprule
Constraint & Fair difference $\hat\Delta$ & Interpretation \\
\midrule
$W_2$ & $\hat\Delta=\hat\Delta^*$ if $\hat D\le\epsilon$; otherwise $\hat\Delta=(\epsilon/\hat D)\hat\Delta^*$. & global shrinkage \\
$W_1$ & $\hat\Delta=\mathcal S_{\hat\lambda}(\hat\Delta^*)$. & soft-thresholding \\
$W_\infty$ & $\hat\Delta(u)=\min\{\epsilon,\max\{-\epsilon,\hat\Delta^*(u)\}\}$. & clipping \\
\bottomrule
\end{tabularx}
\end{table}

The adjusted group curves preserve the empirical weighted average:
\[
    \hat Q_1=\hat{\bar Q}+\hat r_2\hat\Delta,
    \qquad
    \hat Q_2=\hat{\bar Q}-\hat r_1\hat\Delta.
\]
After this adjustment, the last step maps each unconstrained score through its
adjusted group quantile function, as described at the end of this section.

\subsubsection{Known-cost target and its grid approximation}
\label{sec:estimation_pgd}
We now consider convex induced costs beyond squared loss and settings with
more than two groups. When the induced cost is available, we replace the
unknown group weights and latent quantiles in the population problem by their
sample estimates. We then minimize the resulting objective over the full
feasible class
$\mathcal K_{p,\epsilon,\mathcal I}$ of quantile-function tuples on
$u\in(0,1)$. We call a minimizer of this full quantile-function problem the
known-cost plug-in estimator. Its quantile functions may have jumps, and the
optimization uses every rank rather than only a finite grid. Under the
regularity assumptions in Section~\ref{sec:theory}, we establish convergence
rates for this estimator. For numerical work, we approximate the same problem
on a finite midpoint grid by projected-gradient descent.

This approach requires the induced cost
\[
\mathcal C(z,q)=\mathbb E\{\ell(Y,q)\mid f^*(X,S)=z\}
\]
or at least its derivative \(\partial_q\mathcal C(z,q)\) to be available in
closed form. Examples include squared loss and canonical generalized linear
models with a known conditional model. When this conditional average cannot
be calculated, the observed-loss method uses the loss gradient from the
response that remains paired with each ordered score. This alternative is
stated in Remark~\ref{rem:unknown_cost}.

For this finite-grid approximation, fix $M$ and let
$u_k=(k-0.5)/M$, $k=1,\ldots,M$. For each group, define
\[
    \hat q_{s,k}^*=\hat Q_s^*(u_k),
    \qquad k=1,\ldots,M,
\]
and initialize
$\mathbf q_s^{(0)}=(\hat q_{s,1}^*,\ldots,\hat q_{s,M}^*)^\top$. Let
\(\mathcal K\) denote the feasible set in \eqref{Final-Opt}; its elements are
quantile-function tuples defined at every rank. The known-cost plug-in
estimator minimizes
\begin{equation}\label{eq:known_cost_objective}
    \widehat G_{n,\hat r}^C(\mathbf Q)
    :=
    \sum_{s=1}^K \hat r_s\int_0^1
    \mathcal C\!\big(\hat Q_s^*(u),Q_s(u)\big)\,du,
    \qquad \mathbf Q\in\mathcal K .
\end{equation}
The grid approximation replaces each curve by its values at the midpoint
nodes and minimizes
\begin{equation}\label{eq:disc_known_cost_plugin}
    \widehat F_n^C(\mathbf q_1,\dots,\mathbf q_K)
    :=
    \frac1M\sum_{s=1}^K \hat r_s\sum_{k=1}^M
    \mathcal C(\hat q_{s,k}^*,q_{s,k})
    \quad\text{over}\quad
    \mathcal K_{p,M}\cap\mathcal M^K\cap\mathcal B_{\mathcal I} ,
\end{equation}
with the sets defined in \eqref{eq:discretized_prob}. Here
$\mathcal K_{p,M}$ enforces every pairwise fairness bound.
$\mathcal M^K$ keeps each group vector nondecreasing, and
$\mathcal B_{\mathcal I}$ keeps each grid value in the prediction range
$\mathcal I=[a,b]$.

Group $s$ contributes to the objective in proportion to its estimated
frequency $\hat r_s$. The weighted inner product for the projection is
\[
    \langle\mathbf a,\mathbf b\rangle_{\hat r,M}
    :=\frac1M\sum_{s=1}^K\hat r_s\,\mathbf a_s^\top\mathbf b_s,
\]
for $\mathbf a_s,\mathbf b_s\in\mathbb R^M$, $s=1,\ldots,K$.
It uses the same group weights as the objective.
With this choice, the gradient component for group $s$ is
$\partial_q\mathcal C$ itself because the group weight $\hat r_s$ is already
included in the inner product. The projection below uses the same weighted
notion of distance.

For integral constraints $1\le p<\infty$, one projected-gradient iteration
has two steps. First, take a step in the negative-gradient direction for the
known-cost objective:
\begin{equation} \label{eq:known_cost_gradient_step}
    \tilde q_{s,k}^{(t+1)}
    =q_{s,k}^{(t)}-\eta_t\,\partial_q\mathcal C(\hat q_{s,k}^*,q_{s,k}^{(t)}),
    \qquad s=1,\ldots,K,\quad k=1,\ldots,M .
\end{equation}
Second, return the updated grid values to the set that satisfies all fairness,
ordering, and prediction-range requirements:
\begin{equation} \label{eq:weighted_projection_step}
    (\mathbf q_1^{(t+1)},\ldots,\mathbf q_K^{(t+1)})
    =\argmin_{(\mathbf v_1,\ldots,\mathbf v_K)\in
    \mathcal K_{p,M}\cap\mathcal M^K\cap\mathcal B_{\mathcal I}}
    \sum_{s=1}^K \hat r_s\|\mathbf v_s-\tilde{\mathbf q}_s^{(t+1)}\|_2^2.
\end{equation}
The two steps play different roles. The first is a loss-based gradient step.
The second enforces all constraints together. We compute this second step with
weighted Dykstra projection \citep{bauschkecombettes2017}.
Dykstra's method repeatedly
projects onto one constraint at a time and keeps a correction for the effect
of earlier projections. Pairwise grid $L_p$ balls enforce fairness. Isotonic
projection, which replaces a vector by its closest nondecreasing vector,
enforces the quantile order within each group. Coordinate-wise clipping keeps
all values inside $\mathcal I$. The method stores a separate correction for
every pairwise set, every group ordering constraint, and the box constraint.
It therefore computes a projection onto the full intersection of the
constraints.

For two groups, the fairness projection has the same average--difference
geometry as the closed-form squared-loss solution. Let
\[
    \tilde{\bar{\mathbf q}}=\hat r_1\tilde{\mathbf q}_1+\hat r_2\tilde{\mathbf q}_2,
    \qquad
    \tilde{\mathbf d}=\tilde{\mathbf q}_1-\tilde{\mathbf q}_2.
\]
The grid fairness bound requires the adjusted difference to lie in an
$\ell_p$ ball of radius $\epsilon'=M^{1/p}\epsilon$. Project
$\tilde{\mathbf d}$ onto this ball and call the result
$\mathbf d_{\mathrm{proj}}$. The projected group vectors are
\[
    \mathbf q_1=\tilde{\bar{\mathbf q}}+\hat r_2\mathbf d_{\mathrm{proj}},
    \qquad
    \mathbf q_2=\tilde{\bar{\mathbf q}}-\hat r_1\mathbf d_{\mathrm{proj}}.
\]
For $p=2$, the projection rescales the whole difference vector when necessary.
For $p=1$, it uses the standard soft-thresholding projection onto an
$\ell_1$ ball \citep{duchi2008efficient}. With $K>2$, each pair $(i,j)$ uses
the same calculation after normalizing its two weights to
$\hat r_i/(\hat r_i+\hat r_j)$ and
$\hat r_j/(\hat r_i+\hat r_j)$. Dykstra's method cycles through these
pairwise calculations and the other constraints, carrying the corresponding
correction variables from one cycle to the next.

For $p=\infty$, the fairness bound applies separately at each grid point.
One option is to solve the sample version of \eqref{eq:pointwise_convex} at
each point:
\[
    \min_{q_1,\dots,q_K\in\mathcal I}
    \sum_{s=1}^K \hat r_s\mathcal C(\hat q_{s,k}^*,q_s)
    \quad\text{s.t.}\quad |q_i-q_j|\le\epsilon,\quad i,j\in\mathcal S .
\]
The resulting values may not be ordered across grid points. Sorting each group
curve by the simultaneous increasing rearrangement restores the required
order. {This sorting cannot make any pair of curves farther apart.
Assumption~\ref{assum:cost_function2} also ensures that it does not raise the
total loss.} A second option is to use the same
full-intersection grid projection as above, replacing the finite-$p$ balls by
sets that bound the maximum gap.

\paragraph{Examples of known induced costs.}
The induced cost can often be computed from a familiar conditional model. For
squared loss, if $\ell(y,q)=(y-q)^2$ and
$f^*(X,S)=\mathbb E[Y\mid X,S]$, then
\[
    \mathcal C(z,q)=(z-q)^2+\operatorname{Var}(Y\mid f^*(X,S)=z),
    \qquad
    \partial_q\mathcal C(z,q)=2(q-z).
\]
The conditional variance does not depend on $q$, so it does not affect the
fairness adjustment. For logistic loss on the logit scale,
\[
    \ell(y,q)=-yq+\log(1+e^q),
    \qquad \tau(t)=\frac{e^t}{1+e^t},
\]
if \(Y\mid f^*(X,S)=z\sim\mathrm{Bernoulli}(\tau(z))\), then
\[
    \partial_q\mathcal C(z,q)
    =\mathbb E[\tau(q)-Y\mid f^*(X,S)=z]
    =\tau(q)-\tau(z).
\]
More generally, consider a generalized linear model (GLM) with its canonical
link.
If its negative log-likelihood is \(\ell(y,q)=b(q)-yq\), with $q$ on the
natural-parameter scale, and its conditional mean is
\(\mathbb E[Y\mid f^*(X,S)=z]=b'(z)\), then
\[
    \partial_q\mathcal C(z,q)=b'(q)-b'(z).
\]
For logistic and other generalized linear-model losses, the theory requires
the objective to rise sufficiently quickly as the prediction moves away from
its minimum. This property is called strong convexity. The conditions in
Section~\ref{sec:theory} apply on a compact natural-parameter range or after
adding a small ridge regularizer.

\begin{remark}[When the induced cost $\mathcal{C}$ is unknown]
\label{rem:unknown_cost}
{\rm
In many applications the conditional law \(Y\mid f^*(X,S)=z\) is not available, so the induced cost cannot be evaluated directly. A natural practical proxy is to keep the concomitant response
\[
    Y_{s,k}^{\mathrm{obs}}:=Y^{(s)}_{(\lceil n_su_k\rceil)}
\]
attached to the ordered fitted value and replace the known-cost gradient in \eqref{eq:known_cost_gradient_step} by
\begin{equation}\label{eq:observed_loss_gradient_proxy}
    \tilde q_{s,k}^{(t+1)}
    =
    q_{s,k}^{(t)}-\eta_t\,\partial_q\ell(Y_{s,k}^{\mathrm{obs}},q_{s,k}^{(t)}).
\end{equation}
Equivalently, this proxy optimizes the grid objective
\[
    \widehat F_n^L(\mathbf q_1,\dots,\mathbf q_K)
    :=
    \frac1M\sum_{s=1}^K \hat r_s\sum_{k=1}^M
    \ell(Y_{s,k}^{\mathrm{obs}},q_{s,k})
\]
over the same full intersection. This computational proxy is used in the
numerical section, but it is not the estimator covered by the statistical
theorem below.}
\end{remark}

\paragraph{Stage 4: Prediction for new observations.}
The closed-form rule and the full quantile-function problem return complete
adjusted quantile curves. The two grid methods return only adjusted values
$\mathbf q_s$ at the midpoint nodes. We turn these grid values into a curve
after the optimization. For $0<u_1<\cdots<u_M<1$, define
$\operatorname{Interp}_M$ as the interpolation operator that takes a vector
$\mathbf v=(v_1,\ldots,v_M)$ and returns a function on $[0,1]$. The function
$\operatorname{Interp}_M(\mathbf v)$ linearly connects $(u_k,v_k)$ and
$(u_{k+1},v_{k+1})$. It equals $v_1$ on $[0,u_1]$ and $v_M$ on
$[u_M,1]$. We therefore set
$\hat Q_s(u)=\{\operatorname{Interp}_M(\mathbf q_s)\}(u)$ for $u\in[0,1]$. This
interpolation preserves the nondecreasing order and keeps the curve in
$\mathcal I$.

This interpolation step is separate from the empirical quantile construction
in Stage 2. There, $\hat Q_s^*$ is the exact inverse empirical distribution
function. Here, $\hat Q_s$ is the interpolated output of a grid calculation.
In Section~\ref{sec:theory}, $\hat Q_s$ instead denotes the known-cost
plug-in estimator defined by \eqref{eq:known_cost_objective}.

Once an adjusted curve is available, either directly or by interpolation, map
a new score to its within-group rank. The final predictor is
\[
    \hat f_{\mathrm{fair}}(x,s)=\hat Q_s\{\hat F_s(\hat f(x,s))\},
\]
where $\hat F_s$ is the empirical distribution function of the unconstrained
calibration scores in group $s$:
\[
    \hat F_s(t)=n_s^{-1}\sum_{i:S_i=s}\mathbbm 1\{\hat f(X_i,S_i)\le t\}.
\]
We compute this function once from the group-$s$ calibration scores and then
hold it fixed for new observations. Its value can equal $0$ or $1$. The
constant pieces on $[0,u_1]$ and $[u_M,1]$ therefore make the rank map
well-defined at both endpoints and keep every prediction inside the allowed
range.

The two-group ACS implementation in Section~\ref{sec:real_data} uses one
additional convention for its $W_1$ and $W_2$ adjustments. If all adjusted
group quantile vectors are exactly equal to their original vectors, it
returns a copy of the original predictions; otherwise it applies the rank
map above. This avoids changing predictions through an empirical
distribution-function/quantile round trip when no curve adjustment is
needed. This convention is specific to the two-group ACS implementation; it
does not change the simulation, classification, or three-group results, or
the full-function known-cost estimator studied in Section~\ref{sec:theory}.

{The final prediction rule uses the person's group label. The grid
calculation guarantees only that the adjusted curves meet the tolerance at
the grid points. After we convert a fitted score to an estimated rank and
read a value from the curve, the resulting predictions need not satisfy
exactly the same Wasserstein bound. This difference can arise when groups
have different sample sizes or when fitted scores are tied.}

\section{Theoretical Properties} \label{sec:theory}
In this section, we study the statistical
behavior of the proposed estimator, which replaces the unknown population group weights and
quantile functions with estimates from the data. It then solves the same
fairness problem using these estimated inputs.
We obtain two main results. First, we show that the estimated fair quantile functions
converge to their population targets. Second, we show that the final
prediction rule converges to the population fair predictor.

The proof has three steps. Lemma 4 shows that the estimated group quantiles
converge to the true group quantiles. Theorem 2 shows that the fairness adjustment
preserves the convergence rate.
Crucially, strong convexity makes this possible because
it prevents the optimizer from changing too much when the inputs change.
Lemma 5 and Corollary 1 then show that using estimated ranks still produces
consistent predictions.

\subsection{Assumptions}

We now state the assumptions used in our analysis.


\begin{assumption}[Positive group proportions]\label{assum:group_proportions}
The number of sensitive groups \(K\) is fixed and
\[
    r_{\min}:=\min_{1\le s\le K}r_s>0.
\]
For the realized calibration sample, \(n_s=\sum_{i=1}^n\mathbbm 1\{S_i=s\}\) and \(\hat r_s=n_s/n\) are fixed inputs to the empirical optimization. Under the i.i.d. sampling, \(\max_s|\hat r_s-r_s|=O_p(n^{-1/2})\).
\end{assumption}

\begin{assumption}[Density of latent predictors] \label{assum:density}
For each group $s \in \{1, \dots, K\}$, the cumulative distribution function $F^*_s$ of the latent predictor $f^*(X, S) \mid S=s$ admits a probability density function $g_s$ that is bounded away from zero and infinity on a compact support interval $\mathcal I_s$. That is, there exist constants $\underline{c}, \bar{c} > 0$ such that $\underline{c} \le g_s(t) \le \bar{c}$ for all $t \in \mathcal I_s$.
\end{assumption}

\begin{assumption}[Conditional first-stage $L_2$ rate]\label{assum:first_step}
The estimator $\hat f_m$, trained on the independent auxiliary sample of
size $m$ and clipped to $\mathcal I$, satisfies, for a deterministic
$\rho_m\to0$ and every group $s$,
\[
\mathbb E\!\left[(\hat f_m(X,s)-f^*(X,s))^2
\mid\mathcal D_m^{\mathrm{tr}},S=s\right]=O_p(\rho_m^2).
\]
\end{assumption}

The lower density bound in Assumption \ref{assum:density} prevents the quantile curve $Q_s^*$ from changing too
quickly.
Assumption~\ref{assum:first_step}
controls the average squared difference between fitted and true scores within
each group, conditional on the training data.

\subsection{Convergence of Latent Quantiles}
We first study the estimated score quantiles before applying any fairness
adjustment.

\begin{lemma} \label{lem:latent_convergence}
Under Assumptions \ref{assum:group_proportions}, \ref{assum:density} and \ref{assum:first_step}, the empirical latent quantile functions satisfy:
$$ \|\hat{Q}_s^* - Q_s^*\|_2 = O_p\left(\rho_m+n^{-1/2}\right). $$
\end{lemma}

The term $\rho_m$ comes from fitting the regression function, while
$n^{-1/2}$ comes from estimating the score distributions. We next study how
these errors affect the fair quantile functions.

\subsection{Convergence of the Fair Predictor}

We now compare the population fairness problem with its finite-sample version.
Both use the feasible set
\[
\mathcal K\equiv\mathcal K_{p,\epsilon,\mathcal I},
\]
where, for \(p=\infty\), distance means the largest difference after ignoring
sets of ranks with measure zero, a convention called the essential supremum.
The set is bounded and convex.
Both problems minimize a loss over this set. For a
candidate tuple of group quantile functions \(\mathbf Q=(Q_1,\ldots,Q_K)\),
the population loss is
\[
G(\mathbf Q):=\sum_{s=1}^K r_s\int_0^1 \mathcal C\!\big(Q_s^*(u),Q_s(u)\big)\,du.
\]
For each group, the integral averages the induced cost over ranks, and \(r_s\)
weights that group by its population share. To describe the optimizer, we
consider small changes that keep the quantile curves inside \(\mathcal K\).
At the minimum, no allowed direction can reduce \(G\) to first order. This
condition is called a variational inequality\citep{showalter2009hilbert}.

Our main result bounds the distance between the population fair quantiles and
the estimated fair quantiles $\hat Q_s$. Here and below, the
\emph{known-cost plug-in estimator} means a minimizer of
$\widehat G_{n,\hat r}^C$ over the full feasible set.

The next assumption has three parts. First, strong convexity means that the
objective rises sufficiently quickly as we move away from its minimum. This
makes the target unique and stable. Second, the gradient must change smoothly
with the latent score, so an error in the score quantiles produces a controlled
gradient error. Third, the latent-score and prediction ranges are compact, and
the gradient is assumed to be uniformly bounded on that domain.

\begin{assumption}[Regularity of Known Cost] \label{assum:cost_regularity}
The induced expected cost
\[
\mathcal{C}(z, q) := \mathbb{E}[\ell(Y, q) \mid f^*(X,S)=z],
\]
or its derivative \(\Psi(z,q):=\partial_q\mathcal C(z,q)\) is available analytically, up to terms that do not depend on \(q\). It satisfies the following conditions:
\begin{enumerate}
    \item \textbf{Strong Convexity of Induced Cost:}
    For every latent-score value \(z\), the map \(q\mapsto\mathcal C(z,q)\) is \(\mu\)-strongly convex on the feasible prediction range. That is, there exists $\mu > 0$ such that for all feasible $q,q'$ and all relevant $z$,
    \[
    (\partial_q \mathcal{C}(z, q) - \partial_q \mathcal{C}(z, q'))(q - q')
    \ge \mu (q-q')^2.
    \]
    Consequently, the population objective $G$ is \(\mu\)-strongly convex on \(\mathcal K\), and the known-cost plug-in objective is strongly convex on the event \(\min_s n_s>0\).

    \item \textbf{Lipschitz Continuity of the Score in $z$:}
    The partial derivative \(\Psi(z,q)=\partial_q \mathcal{C}(z, q)\) is
    \(L_z\)-Lipschitz in $z$, uniformly over feasible $q$:
    there exists $L_z < \infty$ such that for all feasible $q$ and all relevant $z,z'$,
    \[
    |\partial_q \mathcal{C}(z, q) - \partial_q \mathcal{C}(z', q)| \le L_z |z - z'|.
    \]

    \item \textbf{Bounded feasible range:}
    The feasible set $\mathcal K$ is uniformly bounded in the sense that there exists a compact interval $\mathcal I\subset\mathbb R$ with
    \[
    \{Q_s(u):\mathbf Q\in\mathcal K,\ u\in(0,1]\}\subset\mathcal I,
    \qquad s=1,\dots,K.
    \]
    In addition, \(\Psi\) is finite and uniformly bounded on the relevant compact domain:
    \[
    \max_{1\le s\le K}\ \sup_{z\in\mathcal I_s,\ q\in\mathcal I}
    |\Psi(z,q)|<\infty.
    \]
\end{enumerate}
For squared error loss under \(f^*(X,S)=\mathbb E[Y\mid X,S]\), \(\partial_q\mathcal C(z,q)=2(q-z)\), so \(\mu=2\) and \(L_z=2\). For logistic and canonical GLM losses, the assumption applies on compact natural-parameter ranges with curvature bounded away from zero, or after adding a small ridge regularizer.
\end{assumption}

Lemma~\ref{lem:latent_convergence} controls the estimated score quantiles. The
regularity conditions above then control the resulting gradients. In
particular, Lemma~\ref{lem:latent_convergence} and
Assumption~\ref{assum:cost_regularity} imply
\[
    \sum_{s=1}^K\int_0^1
    \Psi(\hat Q_s^*(u),Q_s(u))^2\,du
    =O_p(1).
\]
To see why, note that the range of \(Q_s^*\) lies in the compact support
interval \(\mathcal I_s\), and every feasible \(Q_s\) lies in
\(\mathcal I\). The Lipschitz condition in \(z\) then controls the effect of
replacing \(Q_s^*\) by \(\hat Q_s^*\).

The known-cost plug-in estimator solves
\[
\hat{\mathbf Q}\in\argmin_{\mathbf Q\in\mathcal K}\widehat G_{n,\hat r}^C(\mathbf Q),
\qquad
\widehat G_{n,\hat r}^C(\mathbf Q)
=
\sum_{s=1}^K \hat r_s\int_0^1
\mathcal C\!\big(\hat Q_s^*(u),Q_s(u)\big)\,du.
\]
This objective estimates \(G\) in two places. It replaces the population group
weights \(r_s\) by their sample versions \(\hat r_s\). It also replaces the
population score quantiles \(Q_s^*\) by the empirical score quantiles
\(\hat Q_s^*\). The group-weight error is of order \(n^{-1/2}\), the same
order as the calibration part of the quantile error.
Theorem~\ref{thm:consistency} accounts for both sources.

For a tuple \(\mathbf Q=(Q_1,\dots,Q_K)\), its population-weighted
\(L_2\) norm is
\[
\|\mathbf Q\|_2
:=
\left(
\sum_{s=1}^K r_s \int_0^1 Q_s(u)^2\,du
\right)^{1/2}.
\]
Given two tuples of square-integrable
functions, \(\mathbf A=(A_1,\ldots,A_K)\) and
\(\mathbf B=(B_1,\ldots,B_K)\), their weighted inner product is
\[
\langle \mathbf A,\mathbf B\rangle
:=
\sum_{s=1}^K r_s \int_0^1 A_s(u)B_s(u)\,du.
\]

\begin{theorem}\label{thm:consistency}
Suppose Assumptions~\ref{assum:group_proportions}, \ref{assum:density}, \ref{assum:first_step}, and \ref{assum:cost_regularity} hold.
Let $\hat{\mathbf Q}=(\hat Q_1,\dots,\hat Q_K)$ be the known-cost
plug-in minimizer of \(\widehat G_{n,\hat r}^C\) over
\(\mathcal K_{p,\epsilon,\mathcal I}\), and let
$\mathbf Q^{\mathrm{fair}}=(Q_1^{\mathrm{fair}},\dots,Q_K^{\mathrm{fair}})$ denote the population-optimal approximately fair quantile function tuple.
Equivalently, let
\[
\mathbf Q^{\mathrm{fair}}\in\argmin_{\mathbf Q\in\mathcal K}G(\mathbf Q),
\qquad
\hat{\mathbf Q}\in\argmin_{\mathbf Q\in\mathcal K}\widehat G_{n,\hat r}^C(\mathbf Q),
\]
where the same bounded convex feasible set is used in both problems.
Then
\[
\sum_{s=1}^K r_s \|\hat Q_s-Q_s^{\mathrm{fair}}\|_2^2
=
O_p\!\left(\rho_m^2+n^{-1}\right).
\]
\end{theorem}

The squared error rate has the same two sources as
Lemma~\ref{lem:latent_convergence}. Fitting the first-stage regression
contributes \(\rho_m^2\), and using a finite calibration sample contributes
\(n^{-1}\). Strong convexity ensures that these small input errors produce a
small change in the constrained minimizer.

Theorem~\ref{thm:consistency} controls the estimated target quantile functions.
We now turn this result into a statement about predictions. Two facts are
needed. First, the rank computed from the fitted score must approach the rank
of the population score. Second, the population fair quantile must change
continuously when its rank changes slightly. For a function
\(g:(0,1)\to\mathbb R\), its Lipschitz constant is
\[
\operatorname{Lip}(g)
:=
\sup_{u\neq v,\;u,v\in(0,1)}
\frac{|g(u)-g(v)|}{|u-v|},
\]
with the convention that $\operatorname{Lip}(g)=\infty$ if the supremum is infinite.
The corollary below needs only continuity, not
a finite Lipschitz constant.

As defined in Section~\ref{sec:estimation}, the calibration sample is used to
construct \(\hat F_s\), \(\hat Q_s^*\), and \(\hat Q_s\). We collect it
together with the auxiliary training sample in
\[
\mathcal H_n
:=
(\mathcal D_m^{\mathrm{tr}},\mathcal D_n^{\mathrm{cal}})
\]
for all data used to construct the final estimator. Let $X^s$ denote an
independent test covariate drawn from the conditional distribution $X\mid S=s$.
The population and estimated ranks of this test observation are
\[
U_s^*:=F_s^*(f^*(X^s,s)),
\qquad
\hat U_s:=\hat F_s(\hat f(X^s,s)).
\]

\begin{lemma}[Stability of empirical ranks]\label{lem:rank_stability}
Suppose Assumptions~\ref{assum:density} and \ref{assum:first_step} hold, and
suppose the auxiliary training sample used to construct $\hat f$ is independent
of the calibration sample used to construct $\hat F_s$. Then, for each group
with $r_s>0$,
\[
\mathbb E\!\left[
\left\{
\hat F_s(\hat f(X^s,s))-F_s^*(f^*(X^s,s))
\right\}^2
\mid \mathcal H_n
\right]
=o_p(1).
\]
Equivalently,
\[
\max_{s:r_s>0}
\mathbb E\!\left[
\left\{
\hat U_s-U_s^*
\right\}^2
\mid \mathcal H_n
\right]
=o_p(1).
\]
Moreover, if
\[
H_{s,n}(u):=\mathbb P(\hat U_s\le u\mid\mathcal H_n),
\qquad u\in[0,1],
\]
then
\[
\max_{s:r_s>0}\sup_{u\in[0,1]}|H_{s,n}(u)-u|=o_p(1).
\]
\end{lemma}

\begin{assumption}[Continuity of population fair quantiles]\label{assum:continuous_fair_quantiles}
For each group with \(r_s>0\), the population fair quantile function
\(Q_s^{\mathrm{fair}}\) is continuous on \([0,1]\).
\end{assumption}

\begin{remark}\label{rem:continuous_fair_quantiles}
{\rm
Lemma~\ref{lem:rank_stability} verifies the stability of the empirical rank
transformation under the sample-splitting scheme used in
Section~\ref{sec:estimation}. The corollary below does not require a uniform
Lipschitz bound on the estimated interpolants \(\hat Q_s\). It uses only the
monotonicity of the estimated quantile curves, the bounded feasible prediction
range from Assumption~\ref{assum:cost_regularity}, and the continuity of the
population fair quantile functions.
}
\end{remark}

\begin{corollary}[Consistency of the final fair predictor]\label{cor:ffair_consistency}
Suppose Assumptions~\ref{assum:group_proportions}, \ref{assum:density}, \ref{assum:first_step},
\ref{assum:cost_regularity}, and \ref{assum:continuous_fair_quantiles} hold.
Define the population approximately fair predictor
\[
f_{\mathrm{fair}}^*(x,s)
:=
Q_s^{\mathrm{fair}}\!\left(F_s^*(f^*(x,s))\right),
\]
and the plug-in fair predictor
\[
\hat f_{\mathrm{fair}}(x,s)
:=
\hat Q_s\!\left(\hat F_s(\hat f(x,s))\right).
\]
Then
\[
\mathcal E_f(\hat f_{\mathrm{fair}},f_{\mathrm{fair}}^*)
:=
\left\{
\sum_{s=1}^K r_s\,
\mathbb E\!\left[
\left\{
\hat f_{\mathrm{fair}}(X^s,s)
-
f_{\mathrm{fair}}^*(X^s,s)
\right\}^2
\mid \mathcal H_n
\right]
\right\}^{1/2}
=o_p(1).
\]
Consequently, the final fair prediction rule is consistent in integrated
squared prediction error.
\end{corollary}

\section{Numerical Experiments}\label{sec:numerical}

\subsection{Simulation setup}

We conducted a simulation study, which focuses on three key questions: (i) what shape does the fairness-accuracy
frontier have?  (ii) where along the distribution does each rule act? and (iii) how well
does a calibration constraint transfer to new data? We report test MSE for
regression and test accuracy for classification.

{During calibration, we represent each group distribution by a
fitted curve on a numerical grid. We report three quantities. The
\emph{grid gap} is the largest distance between any two fitted curves. The
\emph{grid violation} is the amount by which this gap exceeds
\(\epsilon\), with zero meaning that the numerical constraint is met. These
two quantities describe the fitted curves on the grid, not the predictions
assigned to individual observations. The \emph{test gap} is computed from
the final predictions on the held-out sample. We compare gaps only when they
use the same Wasserstein distance.}

{To put the tolerances on a common scale, we write
\(\epsilon=\tau D_{\max}\), where \(D_{\max}\) is the grid gap before any
adjustment. Most paths use \(0\le\tau\le1\). The classification plot extends
to \(\tau=1.1\) so that its last point lies near the unadjusted fit. At
\(\tau=0\), the fitted curves have zero gap on the grid.} For the two-group regression simulations, $\tau=1$
leaves the calibration quantile curves unchanged. The final test predictions can still differ
slightly from the raw predictions because they are reconstructed through the
empirical rank map. The two-group ACS implementation instead preserves the
raw predictions when the adjustment is inactive, as described in
Section~\ref{sec:real_data}. For the observed-loss grid method, $\tau=1$ only makes
the original score quantiles feasible. The two-group curves use $n=2000$ and
$R=50$ replications.

{Only the two-group regression results average over repeated
samples. The three-group, classification, and ACS examples each use one data
split, so their curves are mainly for illustrations.}

\paragraph{Data-generating process 1 (DGP1): linear mean shift.}
This process creates differences across the full
prediction distribution. The two groups are balanced:
\((\mathbb P(S=0),\mathbb P(S=1))=(0.5,0.5)\).
The feature vector $X\in\mathbb{R}^5$ is Gaussian with covariance
$\Sigma_{jk}=0.5^{|j-k|}$. This first-order autoregressive, or AR(1),
structure makes nearby features correlated. Each group has its own intercept
and slope:
\[
Y = f^*_S(X)+\varepsilon,\qquad
f^*_s(X)=\alpha_s + X^\top \beta_s,\qquad \varepsilon\sim N(0,1),
\]
with $(\alpha_0,\alpha_1)=(1,3)$,
$\beta_0=(2,1,0.5,0,0)^\top$, and
$\beta_1=(1,2,0.5,0,0)^\top$.

\paragraph{Data-generating process 2 (DGP2): tail shift.}
The second DGP uses the same basic design and adds an additional group
difference in the upper tail. In group $S=1$, the regression function receives
the bump
\[
f^*_1(X)=\alpha_1 + X^\top \beta_1 + \gamma_{\mathrm{tail}}\,
\mathbbm{1}\{X_1 > t_{\mathrm{tail}}\},
\]
with $t_{\mathrm{tail}}=1$ and $\gamma_{\mathrm{tail}}=3$.

\paragraph{Implementation.}
A stratified 40/30/30 split forms auxiliary, calibration, and test samples.
We fit separate group-wise random forests on the auxiliary sample, learn the
quantile adjustment on the calibration sample, and evaluate it on the test
sample. Each forest uses 200 trees, depth 10, and minimum leaf size 5. The
quantile calculation uses $M=2000$ midpoint nodes on $[0,1]$, with prediction
range $[-20,20]$, exact
empirical quantiles, right-continuous empirical ranks, and constant endpoint
interpolation.

For the computationally more demanding three-group optimization, we use
$M=100$, range $[-20,20]$, and 25 fairness
tolerances, step size $0.1$, projection tolerance $10^{-10}$, and outer
tolerance $10^{-8}$. Weighted Dykstra projection enforces fairness, ordering,
and range constraints together.

\subsection{Two-group regression results}

Figure~\ref{fig:pareto_combined} shows diminishing returns. A moderate
tolerance recovers much of the accuracy lost under exact parity, whereas
removing the last part of the gap is disproportionately costly. In the DGP1
$W_2$ experiment, for example, $\tau=0.5$ removes almost half of the original
test gap but incurs only about one quarter of the MSE increase seen at
{\(\tau=0\)}. DGP2 shows the same pattern. This is the practical value
of a tolerance: it avoids the steepest part of the fairness--accuracy curve.

{The numerical constraints are satisfied accurately on the calibration
quantile grid. Across
all $6{,}000$ adjusted fits, the largest grid violation is}
$1.75\times10^{-14}$. {A test gap can still be positive at
\(\tau=0\) because the adjustment is learned from the calibration sample and
evaluated on a new sample using estimated ranks. This does not indicate a
failure of the grid calculation.}

\begin{figure}[!htbp]
\centering
\includegraphics[width=1\textwidth]{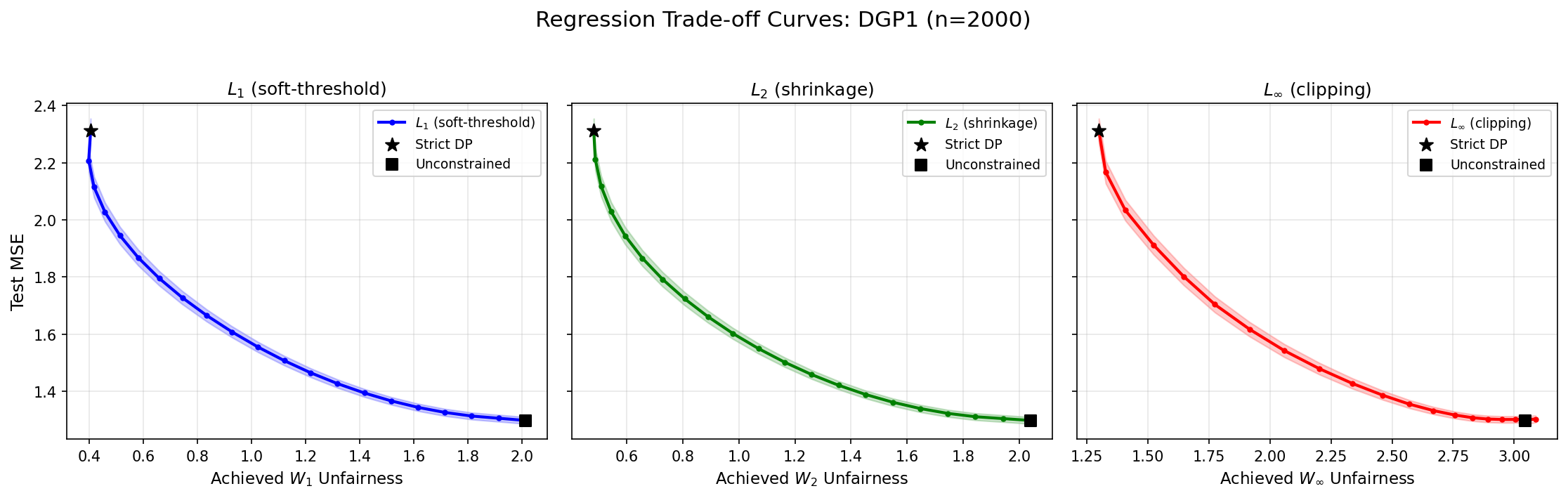}\\[0.3cm]
\includegraphics[width=1\textwidth]{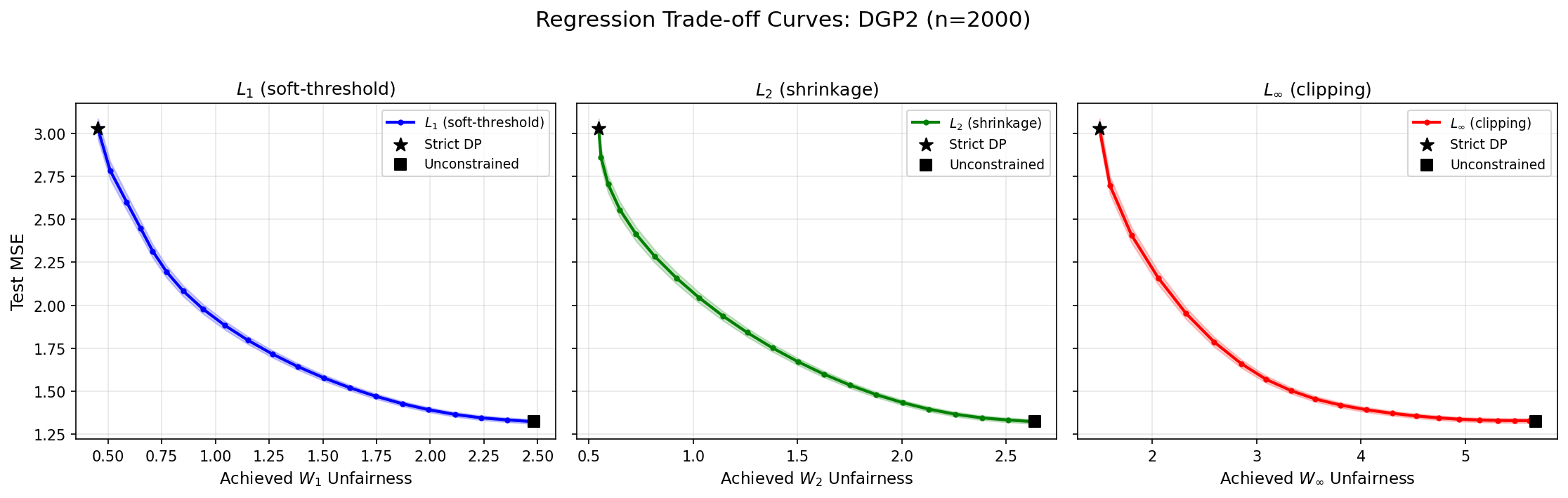}
\caption{Fairness--accuracy frontiers for DGP1 (top) and DGP2 (bottom) at
$n=2000$. The curves become steep near exact parity: the last reduction in
the held-out gap costs more MSE than earlier reductions. Each panel uses its
named Wasserstein norm, so horizontal scales should not be compared. Curves
are Monte Carlo means; ribbons show plus or minus one standard error.}
\label{fig:pareto_combined}
\end{figure}

\begin{table}[!htbp]
\centering
\caption{Simulation results at $\tau=0.5$ (that is,
$\epsilon=0.5D_{\max}$). Entries are means over $R=50$ replications, with
standard errors (SE) in parentheses. ``Test gap'' is the achieved test-set
Wasserstein distance in the norm named by that row.}
\label{tab:sim_tau_half}
\begin{tabular}{lcccc}
\toprule
& \multicolumn{2}{c}{DGP1} & \multicolumn{2}{c}{DGP2} \\
\cmidrule(lr){2-3}\cmidrule(lr){4-5}
Constraint & MSE & Test gap & MSE & Test gap \\
\midrule
$L_1$
  & $1.5550\,(0.0179)$ & $1.0236\,(0.0351)$
  & $1.7159\,(0.0204)$ & $1.2659\,(0.0394)$ \\
$L_2$
  & $1.5493\,(0.0178)$ & $1.0700\,(0.0341)$
  & $1.7519\,(0.0219)$ & $1.3802\,(0.0377)$ \\
$L_\infty$
  & $1.3862\,(0.0159)$ & $2.4619\,(0.0545)$
  & $1.4192\,(0.0147)$ & $3.8044\,(0.0743)$ \\
\bottomrule
\end{tabular}
\end{table}

The quantile plots explain why the adjustment depends on where the groups
differ. In DGP1, the gap extends across most ranks, so every rule changes a
broad part of each distribution. In DGP2, the extra separation is concentrated
in the upper tail. Clipping can therefore leave ranks already inside the
allowed band unchanged and act mainly on the largest gaps; in the
representative plot, about $62\%$ of the ranks remain unchanged.

Table~\ref{tab:sim_tau_half} gives exact results at $\tau=0.5$. The $L_1$ and
$L_2$ rules have nearly the same MSE in DGP1 and differ slightly in DGP2.
This suggests that the location of the disparity can matter, but it does not
establish a general ranking of the norms. Their test-gap values are on
different scales and should be read separately.

\begin{figure}[p]
\centering
\compactquantilegraphic{%
\includegraphics[width=0.95\textwidth]{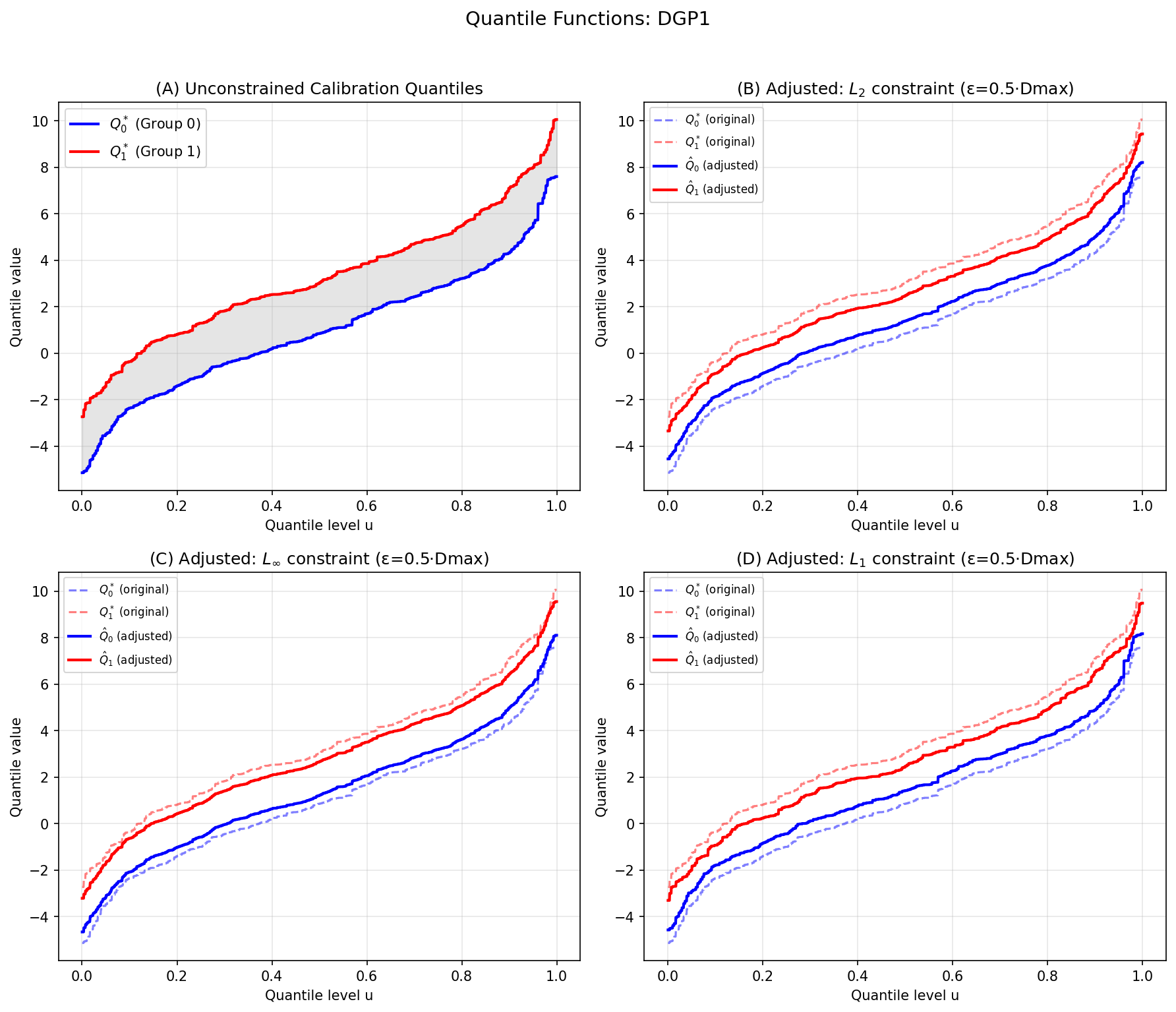}}
\caption{Calibration quantiles for the linear mean-shift process (DGP1) at
$\epsilon=0.5D_{\max}$. Because the groups differ across nearly all ranks,
each rule makes a broad adjustment. Dashed curves are original; solid curves
are adjusted.}
\label{fig:quantiles_dgp1}
\smallskip
\compactquantilegraphic{%
\includegraphics[width=0.95\textwidth]{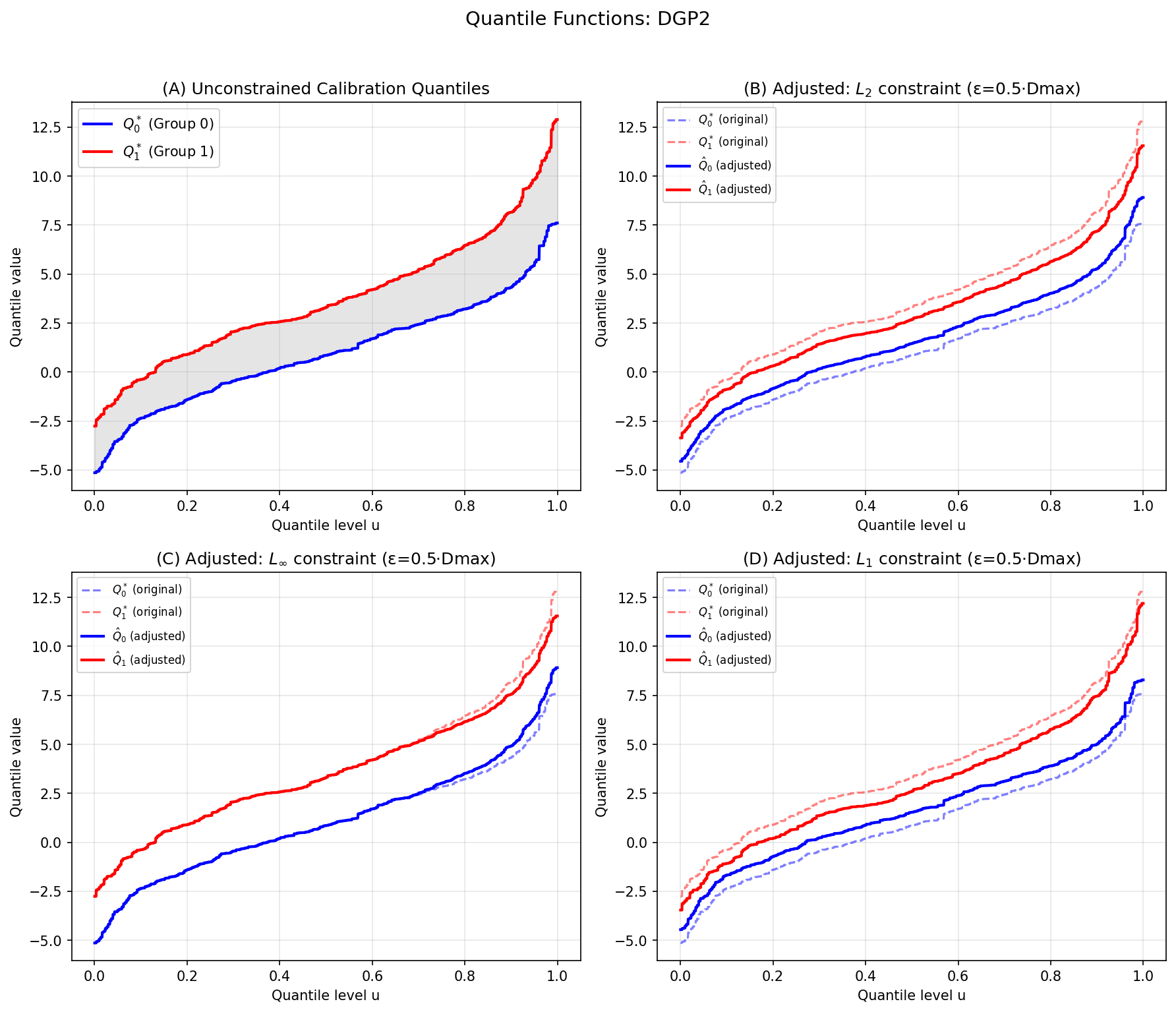}}
\caption{Calibration quantiles for the tail-shift process (DGP2) at
$\epsilon=0.5D_{\max}$. The unconstrained gap grows in the upper tail.
Clipping therefore leaves ranks already inside its band unchanged and acts
mainly where the gap is largest. Dashed curves are original; solid curves are
adjusted.}
\label{fig:quantiles_dgp2}
\end{figure}

\subsection{Multi-group regression (\texorpdfstring{$K=3$}{K=3})}

This experiment tests simultaneous pairwise constraints and unequal group
weights. We use seed 42, $n=3000$, $d=5$, and
\(\mathbb P(S=0,1,2)=(0.2,0.3,0.5)\), so the largest group is two and a
half times as common as the smallest. The outcome follows
\[
Y = \alpha_S + X^\top \beta + \varepsilon,\qquad
(\alpha_0,\alpha_1,\alpha_2)=(1,3,5),\quad
\beta=(1,1,0.5,0,0)^\top,
\]
where $X\sim N(0,I_5)$ and $\varepsilon\sim N(0,1)$.

A stratified 40/30/30 split provides auxiliary, calibration, and test data.
We fit OLS with group indicators on the auxiliary sample, learn the adjustment
on the calibration sample, and evaluate it on the test sample.

\begin{figure}[!htbp]
\centering
\includegraphics[width=0.68\textwidth]{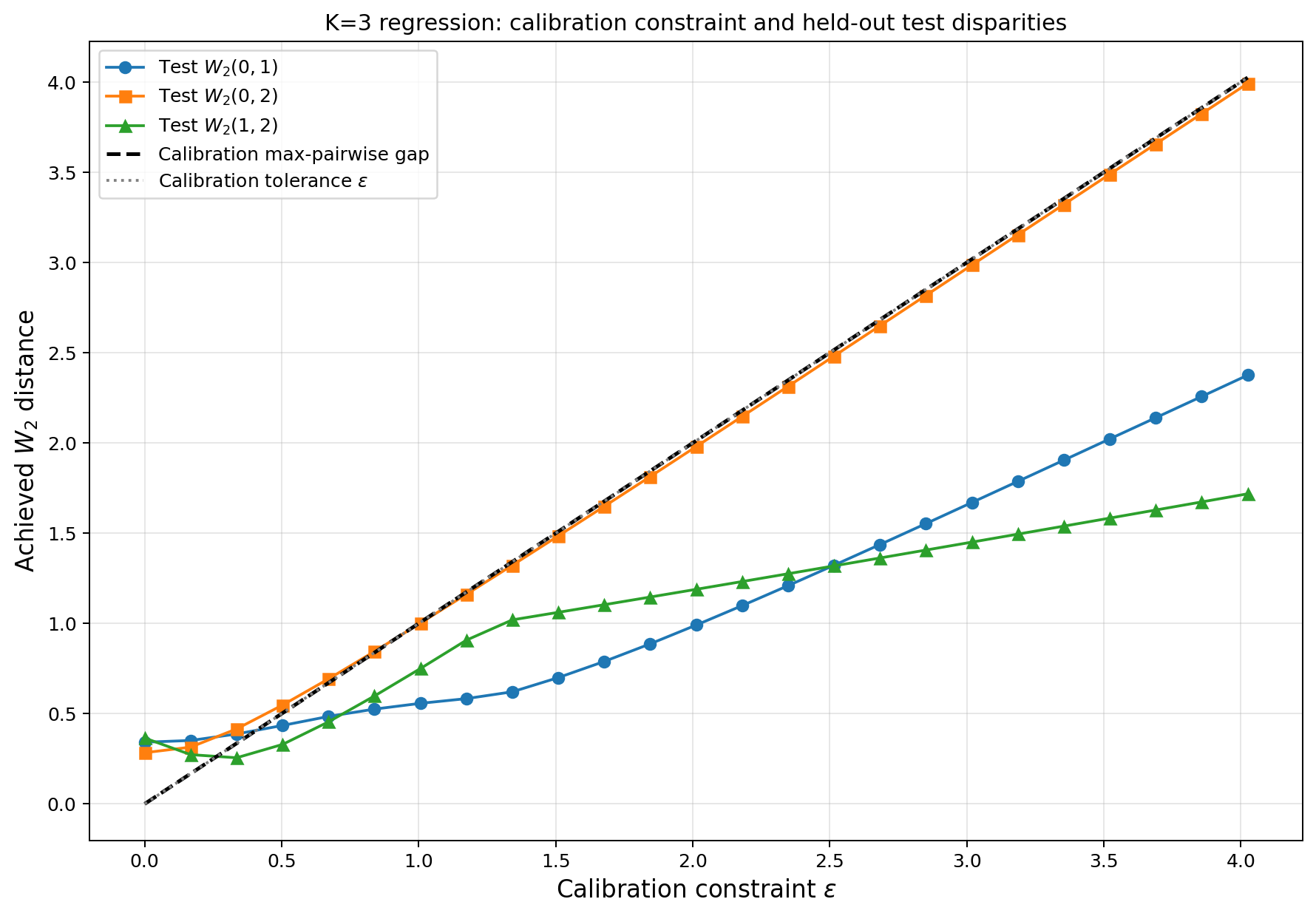}
\caption{Held-out pairwise $W_2$ gaps in the $K=3$ experiment. As the
tolerance grows, the low-versus-high pair becomes the largest held-out gap,
while the other two pairs remain closer. The dashed line is the largest
{gap between the fitted group curves on the calibration grid}; the
dotted line is the calibration tolerance.}
\label{fig:multigroup_distances}
\end{figure}

With three groups, the pair that limits fairness can change with the
tolerance. Under tight constraints, the high-intercept group's distances to
both other groups equal the allowed tolerance. From about $\tau=1/3$ onward,
only the low-versus-high pair reaches that tolerance. The middle group is
already close enough. The method therefore adjusts the pairs that need it
rather than forcing all three pairwise distances to be equal.

The unequal weights also help determine which group absorbs the adjustment. The
realized calibration proportions are $(0.1967,0.2900,0.5133)$. At
$\tau=0.5$, the smallest group's quantile curve moves almost three times as
far as the largest group's curve, as measured by root-mean-square distance
across ranks ($1.58$ versus $0.54$). The group-weighted loss gives the largest
calibration group more influence on the fitted curves. The active pairwise
constraints matter as well. The maximum-pairwise constraint limits how far
apart the groups can be, but it does not require all groups to move by the
same amount.

As in the two-group experiments, most of the predictive improvement occurs
early. The test MSEs are $3.6678$, $1.7769$, and $1.1359$ at
$\tau=0$, $0.5$, and $1$, respectively. Thus, by $\tau=0.5$, the method
recovers about $74.7\%$ of the MSE improvement along the computed path.
Because it minimizes observed calibration losses on a grid, the path need not
end at the raw OLS fit. {All 25 runs stopped only after meeting the
stated numerical criteria. The largest grid violation is}
$3.73\times10^{-11}$.
{The positive test gap at \(\tau=0\) arises because the fitted
adjustment is applied to a new sample using estimated ranks.}

\begin{figure}[!htbp]
\centering
\includegraphics[width=0.72\textwidth]{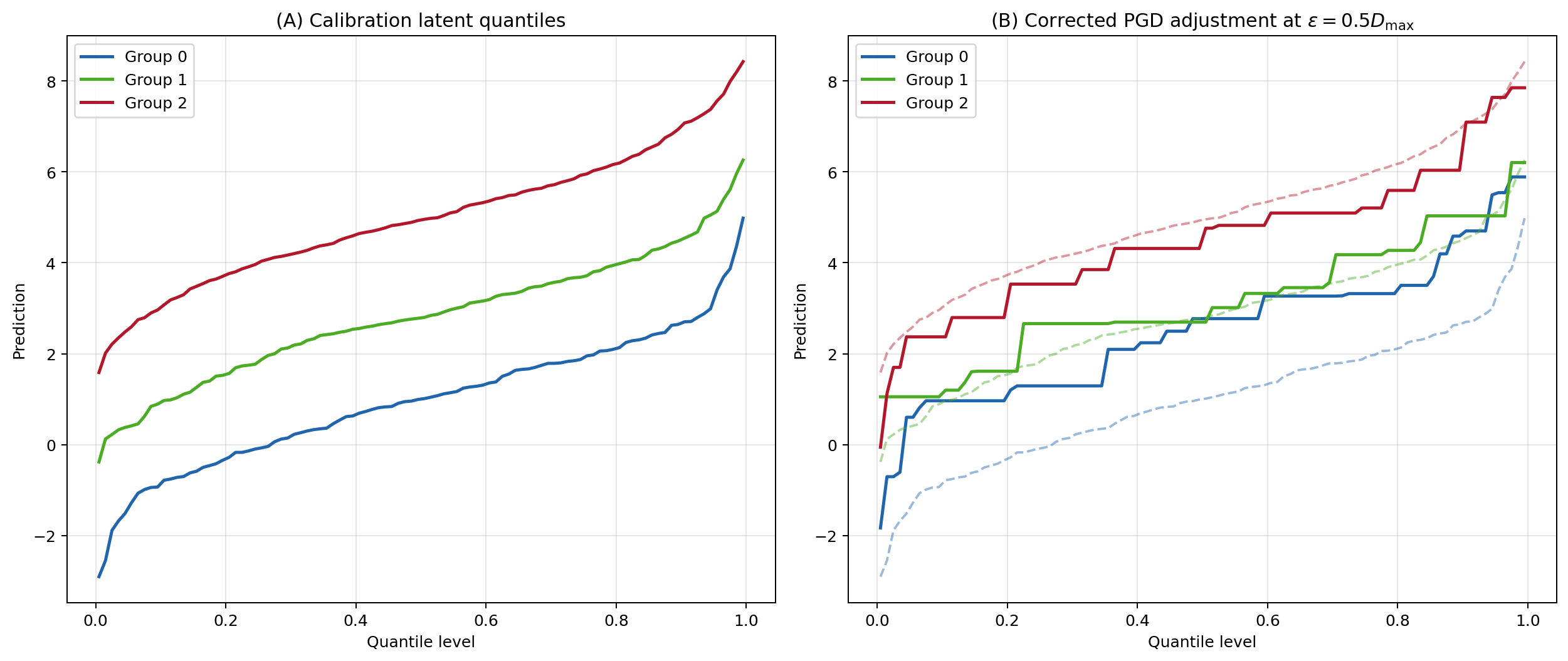}
\caption{Calibration quantiles before and after the $K=3$ adjustment at half
the baseline gap. Panel (A) shows the unadjusted calibration quantile curves.
Panel (B) overlays the unadjusted curves (dashed) and the adjusted curves
(solid).
The smallest group moves most, while the middle and largest groups move much
less. Unequal weights and the active pairwise constraints jointly shape these
changes. Flat stretches arise when the ordering projection pools neighboring
ranks.}
\label{fig:multigroup_quantiles}
\end{figure}

\FloatBarrier
\subsection{Classification illustration}

We next examine whether the same distributional adjustment can reduce
between-group differences in classification scores without a large loss in
accuracy. With seed 42, we generate
$n=4000$ observations from two equally likely groups. We use
$X\sim N(0,I_5)$ and
\[
\mathbb P(Y=1\mid X,S)
=
\operatorname{logit}^{-1}(\alpha_S+X^\top\beta),
\qquad
(\alpha_0,\alpha_1)=(0,2),\quad
\beta=(0.5,0.5,0.3,0,0)^\top.
\]
We use a stratified 40/30/30 split and $M=200$ grid points. We fit logistic
regression using $(X,S)$ on the auxiliary sample, learn a quantile map from
the calibration logits, and apply it to the test logits. We use the
$L_\infty$ clipping rule that is optimal for squared loss and then evaluate
classification accuracy. It is not a logistic-risk optimum.

Figure~\ref{fig:classification_pareto} plots the 30 held-out accuracies
against the {\(W_\infty\) gaps between the adjusted curves on the
calibration grid. At \(\epsilon=0\), this grid gap is zero.} The solid curve
is a fitted line constrained not to decrease, also called an isotonic fit. It
only summarizes the plotted points and does not imply that test accuracy must
increase with the tolerance.

Separate test-gap calculations quantify a useful middle ground. At
$\epsilon=1.1143$ ($\tau\approx0.569$), the held-out logit gap falls from
$1.9300$ to $1.2572$, a $34.9\%$ reduction, while accuracy falls from
$0.7667$ to $0.7458$, a loss of $2.08$ percentage points calculated before
rounding. At exact calibration parity ($\tau=0$),
accuracy is $0.7017$ and the held-out gap is $0.3226$, giving a smaller test
gap but more than three times the loss in accuracy. These held-out gaps are
distinct from the calibration-grid gaps on the figure's horizontal axis.
Accuracy changes in steps when adjusted logits cross the classification
threshold.

\begin{figure}[!htbp]
\centering
\compactclassificationgraphic{%
\includegraphics[width=0.95\textwidth]{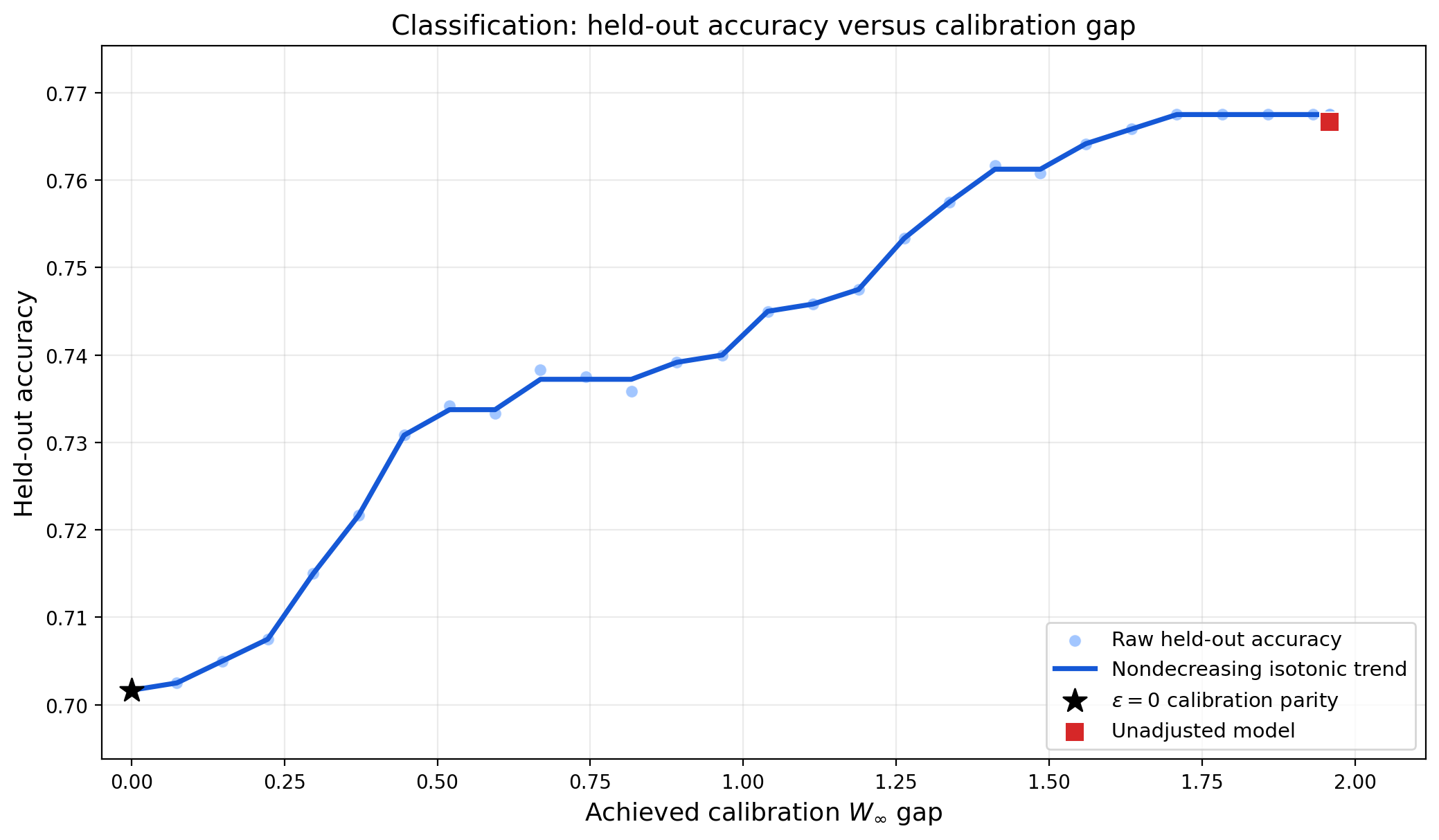}}
\caption{{Calibration-grid \(W_\infty\) gap versus held-out
accuracy. Light points show the 30 observed values. The solid line is a
fitted summary constrained not to decrease. The black star marks
\(\epsilon=0\), where the grid gap is zero, and the red square marks the
unadjusted model at its gap before adjustment. The line does not imply that
held-out accuracy is monotone.}}
\label{fig:classification_pareto}
\end{figure}

\section{Real Data Application: American Community Survey (ACS) Income Prediction}\label{sec:real_data}

\subsection{Data and protocol}
We now move from simulations to real data. This application asks how
prediction accuracy changes when group-specific income-prediction
distributions are moved closer together. We use the
California 2018 one-year ACS person file provided through Folktables
\citep{ding2021retiring}. We ask whether a moderate tolerance preserves much
of the accuracy lost under exact parity and, with three groups, which pair has
the largest held-out gap.

We keep adults older than 18 who have a positive class-of-worker code and
positive personal income. The outcome is $\log(\mathrm{PINCP}+1)$. The nine
predictors are age (AGEP), class of worker (COW), education (SCHL), marital
status (MAR), occupation (OCCP), place of birth (POBP), household relationship
(RELP), usual weekly work hours (WKHP), and sex (SEX). Race is not used as a
predictor; it only defines the groups. The $K=2$ analysis compares White and
non-White adults. The $K=3$ analysis uses values 1, 2, and 6 of the ACS race
field RAC1P for White, Black, and Asian adults. We do not use survey weights.
The results therefore describe the unweighted filtered sample.
{Some ACS predictors describe categories, such as occupation and
marital status. We enter their numerical codes directly into the model.
The trees split on those codes, so arbitrary coding of nominal categories
affects which category groupings a split can express and can affect the fit.
This is a limitation of this illustration.}

The filtered sample contains 207,019 records. Its $K=3$ subset contains
172,850 records and is split into 69,140, 51,855, and 51,855 observations for
the auxiliary, calibration, and test samples. Both analyses use a stratified
40/30/30 split with seed 42. We fit a gradient-boosting regressor on the
auxiliary sample using 200 trees, maximum depth 6, learning rate $0.1$, and
subsample fraction $0.8$. The calibration sample determines the adjustment;
the test sample provides the reported $R^2$ and gaps.

The calibration grid has $M=2000$ midpoint nodes. The prediction maps use ranks
from the right-continuous empirical distribution function and the constant
endpoint extension defined in Section~\ref{sec:estimation}. The fixed
log-prediction range is $\mathcal I=[0,20]$. For $K=3$, we use 25 tolerances,
step size $0.1$, projection tolerance $10^{-10}$, and outer tolerance
$10^{-8}$.

For $K=2$, the $W_1$ and $W_2$ implementations preserve the original
predictions whenever all adjusted quantile vectors exactly equal their
originals. In particular, their $\tau=1$ endpoints coincide with the
unadjusted model. Active adjustments still use the empirical rank map.
The $K=3$ implementation uses the rank map throughout.

For the fair model, let
$\widehat Y_i^{\mathrm{fair}}=\hat f_{\mathrm{fair}}(X_i,S_i)$ denote the
final prediction for test observation $i$. We measure test prediction quality
by $R^2$:
\[
R^2
=
1-\frac{\sum_{i\in \mathcal I_{\mathrm{test}}}(Y_i-\widehat Y_i^{\mathrm{fair}})^2}
{\sum_{i\in \mathcal I_{\mathrm{test}}}(Y_i-\bar Y_{\mathrm{test}})^2},
\qquad
\bar Y_{\mathrm{test}}
=
\frac{1}{|\mathcal I_{\mathrm{test}}|}
\sum_{i\in \mathcal I_{\mathrm{test}}} Y_i.
\]
Here $\mathcal I_{\mathrm{test}}$ indexes the test sample. For the unadjusted
model, we use the same formula with $\hat f(X_i,S_i)$ in place of
$\widehat Y_i^{\mathrm{fair}}$.
Figure~\ref{fig:acs_tradeoff_k2} uses held-out Wasserstein gaps. In
Figure~\ref{fig:acs_tradeoff_k3}, the horizontal axis is the
{largest distance between the fitted group curves on the
calibration grid}, and the vertical axis is held-out $R^2$. At
\(\epsilon=0\), the horizontal value is zero by construction. A nonzero
held-out gap can remain because the adjustment is learned from one finite
sample and applied to another using estimated ranks. It does not mean that
the numerical constraint failed.

\subsection{Results: \texorpdfstring{$K=2$}{K=2} (White vs.\ Non-White)}

The two panels in Figure~\ref{fig:acs_tradeoff_k2} tell the same practical
story. At $\tau=14/29\approx0.48$, the tolerance is about half the baseline
calibration gap and the held-out gap is also about half its baseline value.
The method has recovered $61.6\%$ ($W_1$) and $61.4\%$ ($W_2$) of the
$R^2$ lost under exact calibration parity. The
$W_1$ and $W_2$ adjustments also trace almost identical $R^2$ paths at matched
normalized tolerances. In this example, the tolerance has a clearer effect on
prediction than the choice between these two adjustment rules. Their gap
scales remain different and are not compared.

{At \(\tau=0\),} the corresponding held-out gap falls by more than
$90\%$ in either panel, at an absolute $R^2$ cost of $0.0078$
($1.34\%$ relative). {The largest grid violation is}
$1.11\times10^{-15}$. {This confirms that the fitted curves meet
the constraint on the grid. It does not imply that the predictions assigned
to individual calibration observations have exactly the same empirical
distribution.}

At the inactive endpoint, both methods reproduce the unadjusted model's
$R^2=0.5817$, with held-out gaps $0.1865$ ($W_1$) and $0.1900$ ($W_2$).
The last step of the plotted path rises from $R^2\approx0.58086$ at
$\tau=28/29$ to $0.58167$ at $\tau=1$. At this endpoint, the identity
convention bypasses empirical rank reconstruction. The final rise should
therefore not be attributed entirely to relaxing the fairness constraint.

\begin{figure}[t]
\centering
\includegraphics[width=0.95\textwidth]{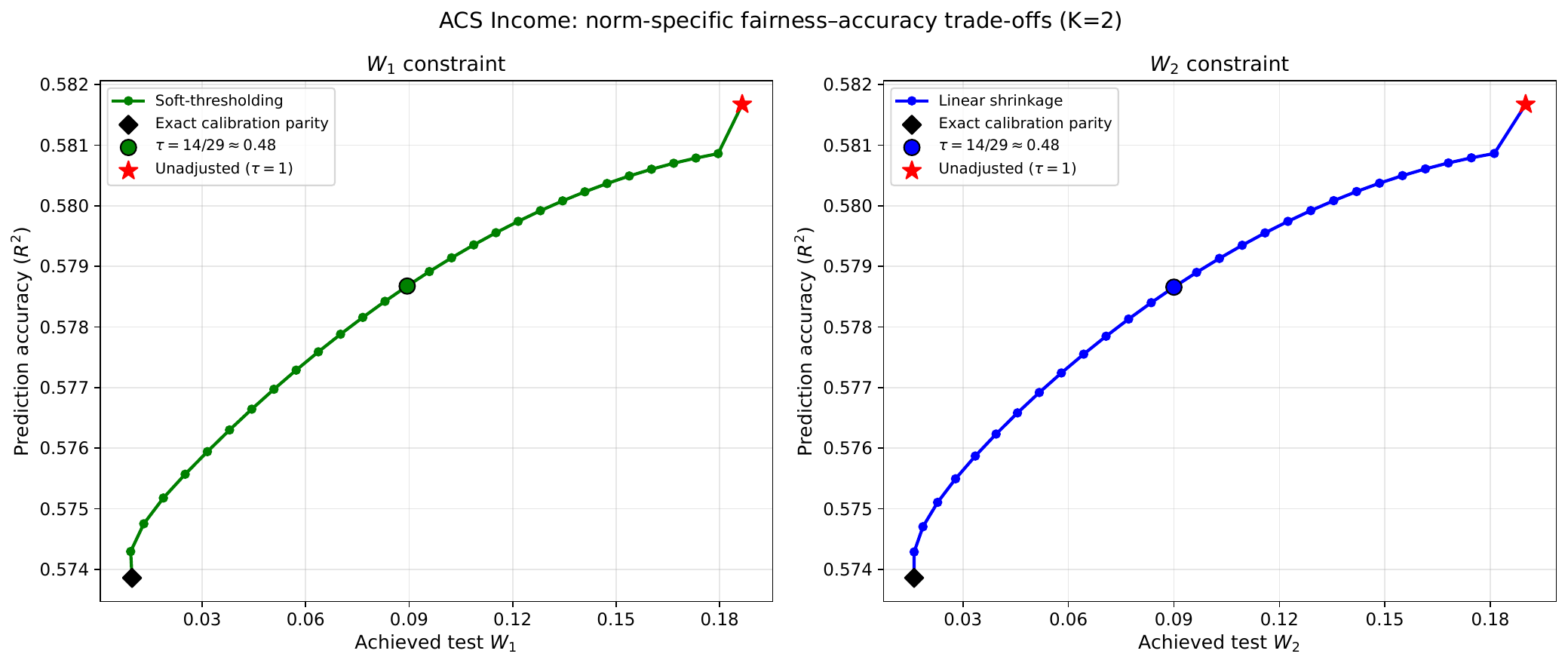}
\caption{California 2018 ACS income regression with $K=2$ groups. The
curves show all 30 observed settings per norm in tolerance order. The marked
near-half setting is $\tau=14/29$. At $\tau=1$, the identity convention
preserves the unadjusted predictions, so the final point coincides with the
unadjusted model. The final rise reflects both the tolerance change and the
bypass of empirical rank reconstruction.
The exact-calibration-parity setting can retain a nonzero held-out gap.
Each panel uses its own Wasserstein norm and should be read separately.}
\label{fig:acs_tradeoff_k2}
\end{figure}

\subsection{Results: \texorpdfstring{$K=3$}{K=3} (White vs.\ Black vs.\ Asian)}

The clearest three-group finding is which pair has the largest held-out gap.
White adults make up about three quarters of the calibration sample. Even so,
the Black--Asian pair has the largest held-out gap at every tolerance. A
method that compared each group only with White would therefore miss the pair
with the greatest separation.

{At \(\tau=0\),} the largest held-out gap falls from $0.2789$ to
$0.0817$, a $70.7\%$ reduction, while $R^2$ falls by $0.0099$
($1.73\%$ relative). Figure~\ref{fig:acs_tradeoff_k3} plots the
25 held-out $R^2$ values against the {largest distances between
the fitted group curves on the calibration grid}. Loosening the constraint
lowers the optimized calibration loss but need not raise held-out $R^2$. The
shallow dip reflects learning the adjustment on one sample and evaluating it
on another. The fitted nondecreasing line is only a visual summary, so the
tolerance should be selected using validation data.

The computed path does not end at the raw unadjusted model because the method
minimizes observed calibration losses on a grid and then applies an estimated
rank map. The distance to the red star is therefore not entirely a fairness
cost. {Across all 25 fits, the largest positive excess of a pairwise
calibration-grid $W_2$ distance above its tolerance is
\(9.46\times10^{-11}\).}

\begin{figure}[!htbp]
\centering
\includegraphics[width=0.58\textwidth]{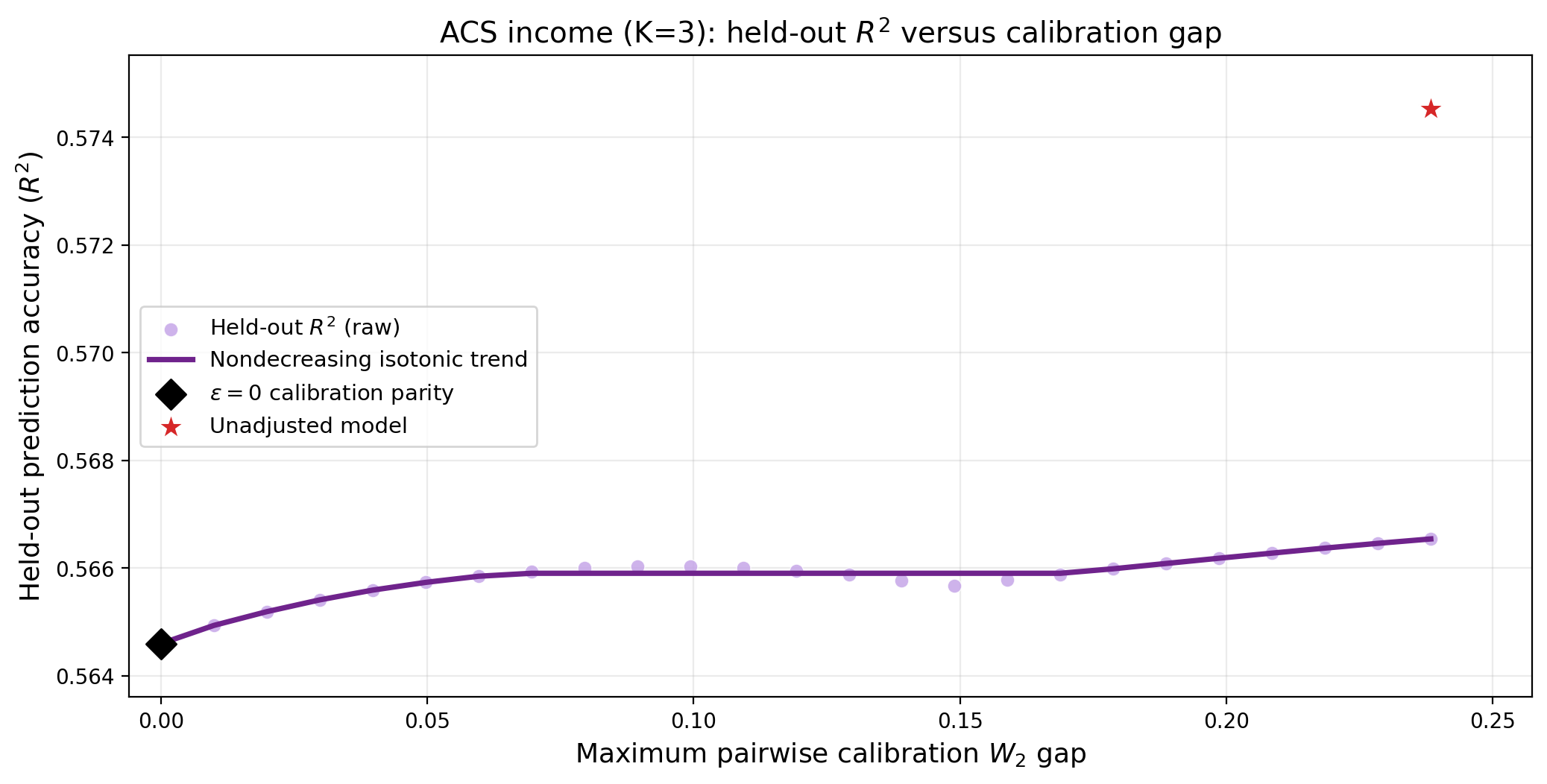}
\caption{{Maximum pairwise \(W_2\) distance between the fitted
group curves on the calibration grid versus held-out \(R^2\). Light points
show the 25 observed values. The solid line is a fitted summary constrained
not to decrease. The black diamond marks \(\epsilon=0\), where the grid gap
is zero, and the red star marks the unadjusted model at its gap before
adjustment. The line does not imply that held-out \(R^2\) is monotone.}}
\label{fig:acs_tradeoff_k3}
\end{figure}

\subsection{Distribution shift visualization}

We next look beyond the largest pairwise gap and compare the full prediction
distributions before and after adjustment. The adjustment changes the shapes
and relative locations of the calibration prediction distributions. The
smoothed group densities remain distinct at this positive tolerance.

\begin{figure}[!htbp]
\centering
\includegraphics[width=0.95\textwidth]{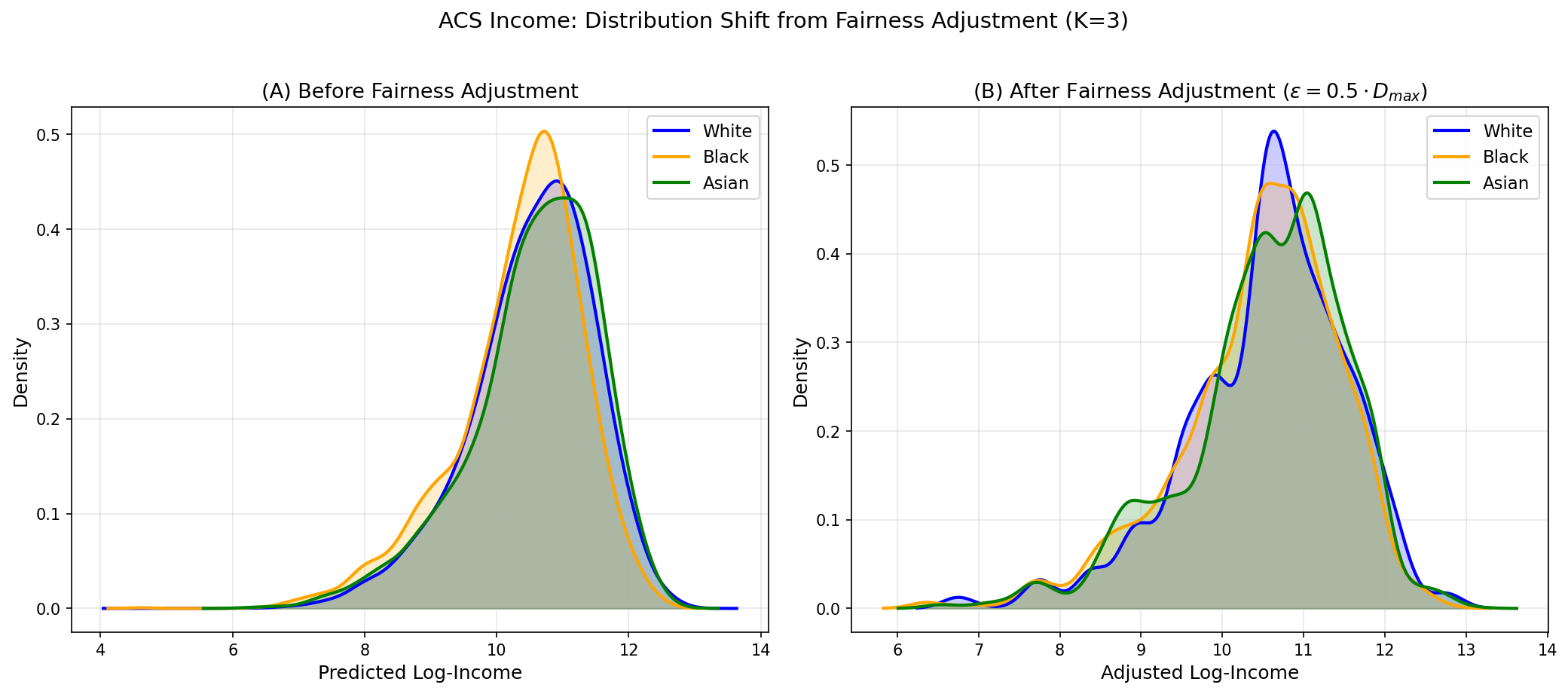}
\caption{Smoothed distributions of the California 2018 ACS calibration
predictions before and after the $K=3$ adjustment at
$\epsilon=0.5D_{\max}$. The smoothed group densities remain distinct at this
positive tolerance.}
\label{fig:income_dist_shifts}
\end{figure}

\section{Conclusion} \label{sec:conclusion}
{We developed a framework for fair regression under
\textit{Approximate Demographic Parity}. By replacing strict independence
with a distance-based constraint, the framework makes the trade-off between
predictive utility and distributional parity explicit and tunable.

Using optimal transport, we show that the prediction problem and the
corresponding problem over group distributions have the same best achievable
value. When the distributional problem has an optimal solution and the
required deterministic transport rules exist, those rules give an optimal
approximately fair predictor. Under the stated one-dimensional monotonicity
condition, the problem can also be written in terms of group-specific
quantile functions. This quantile problem is convex when the induced cost is
convex in the prediction argument. For squared error loss with two
positively weighted groups, we obtained closed-form solutions under
\(W_1\), \(W_2\), and \(W_\infty\) constraints. They correspond to
soft-thresholding, linear shrinkage, and clipping. For known induced costs
that are convex in the prediction argument, including settings with multiple
sensitive groups, we developed numerical methods based on clipping or
weighted projected-gradient updates on a finite grid.

When the induced cost is known, we establish convergence rates for the
estimated fair quantile functions and consistency of the final prediction
rule under the stated assumptions. The simulations show how predictive
accuracy and held-out fairness change as the tolerance varies. We report
satisfaction of the numerical constraints separately from fairness gaps on
held-out data. The California ACS income application further illustrates the
framework and shows that, for this prediction task, a substantial portion of
the held-out Wasserstein gap can be reduced with only a small loss in
predictive accuracy. This interpretation is descriptive for the fitted
prediction problem and should not be read as a causal claim about the sources
of disparity.

Future work includes deriving error bounds for the finite-grid numerical
method, developing analytic solutions for additional multi-group problems,
and extending the transport-based approach to other fairness criteria, such
as equalized odds.}

\section*{Acknowledgments}
The authors used OpenAI Codex with the GPT-5.6-Sol and GPT-6 models for language editing,
LaTeX preparation, code debugging, and mathematical and computational
consistency checks. All mathematical arguments, citations, code, numerical
results, and final text were reviewed and verified by the authors, who take
full responsibility for the manuscript.

\bibliographystyle{chicago}
\bibliography{references_v9}

\appendix
\section{Proofs of the Main Results}
This appendix first proves the population and statistical results. It then
derives the closed-form formulas and the multi-group squared-loss results.
Each proof starts with its main idea. The calculations then show how that idea
gives the stated result.

\begin{proof}[Proof of Theorem \ref{thm1}]
Let $A$ be the infimum of the prediction problem in \eqref{CFP_cost}, and let
$B$ be the infimum of the constrained transport problem in \eqref{KP}. We
prove the two inequalities separately.

\emph{First, \(B\le A\).}
Take any feasible predictor \(f\), and let
\(\nu_{f,s}=\mathcal L(f(X,S)\mid S=s)\). Because \(f\) is feasible, the
distributions \(\nu_{f,1},\ldots,\nu_{f,K}\) satisfy all pairwise Wasserstein
constraints. For group \(s\), define
\[
\gamma_{f,s}
:=
\mathcal L\big(f^*(X,S),f(X,S)\mid S=s\big).
\]
Its first marginal is \(\mu_s\), and its second marginal is \(\nu_{f,s}\).
Thus \(\gamma_{f,s}\in\Pi(\mu_s,\nu_{f,s})\), and
\begin{align*}
B
&\le
\sum_{s=1}^K r_s
\inf_{\gamma_s\in\Pi(\mu_s,\nu_{f,s})}
\int \mathcal C(z,q)\,d\gamma_s(z,q)\\
&\le
\sum_{s=1}^K r_s
\int \mathcal C(z,q)\,d\gamma_{f,s}(z,q)\\
&=
\mathbb E\!\left\{\mathcal C\big(f^*(X,S),f(X,S)\big)\right\}.
\end{align*}
This holds for every feasible \(f\). Taking the infimum over \(f\) gives
\(B\le A\).

\emph{Next, \(A\le B\).}
Fix \(\delta>0\). By the definitions of the outer and inner infima in
\eqref{KP}, we can choose feasible distributions
\(\nu_1,\ldots,\nu_K\) and couplings
\(\gamma_s\in\Pi(\mu_s,\nu_s)\) such that
\[
\sum_{s=1}^K r_s
\int \mathcal C(z,q)\,d\gamma_s(z,q)
\le B+\delta.
\]
Because \(\mu_s\) is non-atomic, the graph-measure approximation used in the
proof of Theorem~2.1 of \citet{ambrosio2003lecture} gives measurable maps
\(T_{s,j}\) satisfying
\[
(T_{s,j})_\#\mu_s=\nu_s,
\qquad
(\operatorname{id},T_{s,j})_\#\mu_s
\ \Longrightarrow\ \gamma_s ,
\]
where \(\Longrightarrow\) denotes weak convergence. The first equality is
important: every approximation has exactly the chosen output distribution,
so it preserves the pairwise fairness constraints. Since \(\mathcal C\) is
continuous on the compact set \(\mathcal I^2\), it is bounded, and weak
convergence gives
\[
\int \mathcal C\big(z,T_{s,j}(z)\big)\,d\mu_s(z)
\longrightarrow
\int \mathcal C(z,q)\,d\gamma_s(z,q).
\]
Define one predictor across all groups by
\[
f_j(x,s):=T_{s,j}(f^*(x,s)).
\]
Its conditional distribution in group \(s\) is \(\nu_s\), so \(f_j\) is
feasible. Because the number of groups is finite,
\begin{align*}
A
&\le
\lim_{j\to\infty}
\mathbb E\!\left\{\mathcal C\big(f^*(X,S),f_j(X,S)\big)\right\}\\
&=
\sum_{s=1}^K r_s
\int \mathcal C(z,q)\,d\gamma_s(z,q)
\le B+\delta.
\end{align*}
Letting \(\delta\downarrow0\) yields \(A\le B\). Therefore \(A=B\).

\emph{Finally, suppose a deterministic optimizer exists.}
Assume that the constrained transport problem attains \(B\) at
\((\bar\nu_1,\ldots,\bar\nu_K)\) and that, for each group, an optimal Monge
map \(T_s\) from \(\mu_s\) to \(\bar\nu_s\) exists. Define
\[
f_{\mathrm{fair}}^*(x,s):=T_s(f^*(x,s)).
\]
This predictor has conditional distribution \(\bar\nu_s\) in group \(s\), so
it is feasible. Moreover,
\begin{align*}
\mathbb E\!\left\{\mathcal C\big(f^*(X,S),
f_{\mathrm{fair}}^*(X,S)\big)\right\}
&=
\sum_{s=1}^K r_s
\int \mathcal C\big(z,T_s(z)\big)\,d\mu_s(z)\\
&=
\sum_{s=1}^K r_s
\inf_{\gamma_s\in\Pi(\mu_s,\bar\nu_s)}
\int \mathcal C(z,q)\,d\gamma_s(z,q)\\
&=B=A.
\end{align*}
Hence \(f_{\mathrm{fair}}^*\) attains the AFRP infimum.
\end{proof}

\begin{proof}[Proof of Lemma \ref{lem:rearrangement}]
Increasing rearrangement means sorting each function into quantile order. We
show that this sorting has two useful effects. It cannot enlarge the
group-to-group distances, and it cannot increase the prediction cost.

\emph{Distances between groups.}
For $p<\infty$, matching equal ranks is the least costly way to pair two
one-dimensional distributions under the cost \(|x-y|^p\). This rank matching
is the monotone coupling. It minimizes
$\mathbb E|h_i(U)-h_j(U)|^p$ among all pairings of the same marginal laws.
Therefore
$\|h_i^\uparrow-h_j^\uparrow\|_p\le\|h_i-h_j\|_p$. The same conclusion for
$p=\infty$ follows from the one-dimensional quantile representation of
$W_\infty$.

\emph{Prediction cost and quantile properties.}
The inequality $\partial^2_{zq}\mathcal C<0$ says that matching a larger score
with a larger prediction is no more costly than crossing the two matches.
This property is called submodularity, or quasi-antitonicity. The
fixed-marginal rearrangement inequality of
\citet[Theorem~2]{cambanis1976inequalities}, applied to $-\mathcal C$, then
shows that matching scores and predictions at the same percentile minimizes
the expected cost for \(\mathcal C\). This is also called a comonotone
coupling. For the marginals of $Q_s^*(U)$ and $h_s(U)$, it is
$(Q_s^*(U),h_s^\uparrow(U))$. This gives the stated cost inequality.

The rearranged functions are quantile functions, so they are nondecreasing
and left-continuous. Sorting changes only the order of the values, not their
distribution. This preservation of the value distribution, called
equimeasurability, keeps the functions integrable. It also keeps their values
in \(\mathcal I\), apart from possible changes on sets of measure zero. This
is the essential-range statement.
\end{proof}

\begin{proof}[Proof of Lemma \ref{lem:monotonicity_closed_form}]
Each closed-form rule changes only the difference between the two groups. The
key fact is that its difference map \(\psi\) preserves order, never expands a
difference, and leaves zero fixed. We use these properties first to prove that
the adjusted curves remain nondecreasing. We then verify their range and
regularity.

\emph{The adjusted curves remain ordered by rank.}
Fix $0<u<v<1$ and write
\[
    a_s=Q_s^*(v)-Q_s^*(u)\ge 0,\qquad s=1,2.
\]
Then
\[
    \Delta^*(v)-\Delta^*(u)=a_1-a_2 .
\]
For $Q_1$, we have
\[
\begin{aligned}
    Q_1(v)-Q_1(u)
    &= r_1 a_1+r_2 a_2
    + r_2\{\psi(\Delta^*(v))-\psi(\Delta^*(u))\}.
\end{aligned}
\]
If $a_1\ge a_2$, then $\Delta^*(v)\ge \Delta^*(u)$. Because $\psi$ is
nondecreasing, the last term is nonnegative. Hence
$Q_1(v)-Q_1(u)\ge 0$. Now suppose $a_1<a_2$. The $1$-Lipschitz property and
monotonicity of $\psi$ imply
\[
    \psi(\Delta^*(v))-\psi(\Delta^*(u))
    \ge \Delta^*(v)-\Delta^*(u)
    = a_1-a_2 .
\]
Therefore,
\[
    Q_1(v)-Q_1(u)
    \ge r_1a_1+r_2a_2+r_2(a_1-a_2)
    = a_1\ge 0.
\]
Thus \(Q_1\) cannot decrease as the rank moves from \(u\) to \(v\).

Similarly,
\[
\begin{aligned}
    Q_2(v)-Q_2(u)
    &= r_1 a_1+r_2 a_2
    - r_1\{\psi(\Delta^*(v))-\psi(\Delta^*(u))\}.
\end{aligned}
\]
If $a_1\le a_2$, then
$\psi(\Delta^*(v))-\psi(\Delta^*(u))\le 0$, so
$Q_2(v)-Q_2(u)\ge 0$. Now suppose $a_1>a_2$. The $1$-Lipschitz property gives
\[
    \psi(\Delta^*(v))-\psi(\Delta^*(u))
    \le \Delta^*(v)-\Delta^*(u)
    = a_1-a_2,
\]
and hence
\[
    Q_2(v)-Q_2(u)
    \ge r_1a_1+r_2a_2-r_1(a_1-a_2)
    = a_2\ge 0.
\]
Thus \(Q_2\) is also nondecreasing.

\emph{The adjusted values stay in range.}
The condition $\psi(0)=0$, together with monotonicity and the $1$-Lipschitz
property, implies that $\psi(t)=\alpha(t)t$ for some
$\alpha(t)\in[0,1]$ (take $\alpha(0)=0$). Here \(\alpha(t)\) is the fraction
of the original group difference that remains after adjustment. Pointwise,
\[
\begin{aligned}
Q_1
&=\{r_1+r_2\alpha(\Delta^*)\}Q_1^*
  +r_2\{1-\alpha(\Delta^*)\}Q_2^*,\\
Q_2
&=r_1\{1-\alpha(\Delta^*)\}Q_1^*
  +\{r_2+r_1\alpha(\Delta^*)\}Q_2^*.
\end{aligned}
\]
The coefficients in each line are nonnegative and sum to one. Both adjusted
values are therefore weighted averages of $Q_1^*(u)$ and $Q_2^*(u)$, so they
remain in $\mathcal I$. Every Lipschitz map is continuous, so composition
with \(\psi\) preserves left continuity. Boundedness then gives
integrability.

\emph{The three adjustment rules satisfy these conditions.}
Linear shrinkage uses $\psi(t)=\alpha t$ with $\alpha\in[0,1]$.
Soft-thresholding uses
$\psi(t)=\operatorname{sgn}(t)(|t|-\lambda)_+$. Clipping uses
$\psi(t)=\operatorname{clip}_{[-\epsilon,\epsilon]}(t)$. Each map is
nondecreasing, $1$-Lipschitz, and zero at zero. Therefore the resulting
$Q_1$ and $Q_2$ have all the asserted properties.
\end{proof}

\begin{proof}[Proof of Lemma \ref{lem:latent_convergence}]
The estimated quantile differs from the population quantile for two reasons.
We use fitted scores instead of true scores, and we observe only a finite
calibration sample. We control these two errors separately and then add them.

\emph{Error from fitted scores.}
Fix $s$. Let $\hat\nu_s$ be the empirical measure of the fitted calibration
scores, $\tilde\nu_s$ the empirical measure of the corresponding unobserved
true scores with quantile $\tilde Q_s^*$, and $\mu_s$ the population
true-score law. Pair each fitted score with the true score for the same
observation. This is one valid way to transport \(\hat\nu_s\) to
\(\tilde\nu_s\), so
\[
W_2^2(\hat\nu_s,\tilde\nu_s)
\le n_s^{-1}\sum_{i:S_i=s}
|\hat f_m(X_i,s)-f^*(X_i,s)|^2=O_p(\rho_m^2),
\]
by independence of training and calibration and Assumption
\ref{assum:first_step}.

\emph{Error from the finite calibration sample.}
Let $\tilde F_s$ be the empirical cumulative distribution function (CDF) of
the true scores. The
Dvoretzky--Kiefer--Wolfowitz inequality controls its largest difference from
the population distribution function. Because every group has positive
probability, $n_s$ has the same order as $n$. Therefore,
$\|\tilde F_s-F_s^*\|_\infty=O_p(n^{-1/2})$. Since the density of $F_s^*$ is
at least $\underline c$, its cumulative distribution function cannot be too
flat. Inverting the two distribution functions therefore gives
$\|\tilde Q_s^*-Q_s^*\|_\infty
\le\underline c^{-1}\|\tilde F_s-F_s^*\|_\infty
=O_p(n^{-1/2})$, using constant endpoint extensions at \(0\) and \(1\).
Combining this bound with the fitted-score bound and applying the triangle
inequality gives
\[
\|\hat Q_s^*-Q_s^*\|_2
=W_2(\hat\nu_s,\mu_s)
\le W_2(\hat\nu_s,\tilde\nu_s)
+\|\tilde Q_s^*-Q_s^*\|_2
=O_p(\rho_m+n^{-1/2}).
\]
\end{proof}

\begin{proof}[Proof of Theorem \ref{thm:consistency}]
The population and estimated problems use the same feasible set but slightly
different inputs. The proof shows that this small input change causes only a
small change in the optimizer. We first write the two objectives and their
gradients. We then use the fact that no feasible move can improve either
optimizer. Strong convexity turns this fact into a stability bound. Finally,
we control the gradient difference using the score-quantile and group-weight
errors.

Our goal is to bound
$\|\hat{\mathbf Q}-\mathbf Q^{\mathrm{fair}}\|_2$. Recall that
$\|\mathbf Q\|_2^2:=\sum_{s=1}^K r_s\int_0^1 Q_s(u)^2\,du$. The corresponding
inner product is defined for two tuples of functions,
\(\mathbf A=(A_1,\ldots,A_K)\) and
\(\mathbf B=(B_1,\ldots,B_K)\), by
\[
\langle \mathbf A,\mathbf B\rangle
:=
\sum_{s=1}^K r_s\int_0^1 A_s(u)B_s(u)\,du.
\]
Write $\mathbf Q^*:=(Q_1^*,\ldots,Q_K^*)$ and
$\hat{\mathbf Q}^{\,*}:=(\hat Q_1^*,\ldots,\hat Q_K^*)$.

\paragraph{Step 1: Write the two objectives and their slopes.}
Define the population objective
\[
G(\mathbf Q):=\sum_{s=1}^K r_s\int_0^1 \mathcal C\!\big(Q_s^*(u),\,Q_s(u)\big)\,du,
\qquad \mathcal C(z,q):=\mathbb E[\ell(Y,q)\mid f^*(X,S)=z],
\]
and the estimated objective
\[
\widehat G_{n,\hat r}^C(\mathbf Q)
:=\sum_{s=1}^K \hat r_s\int_0^1
\mathcal C\!\big(\hat Q_s^*(u),Q_s(u)\big)\,du .
\]
By Assumption~\ref{assum:cost_regularity}(1), $G$ is $\mu$-strongly convex on $\mathcal K$ and \(\widehat G_{n,\hat r}^C\) is convex on the event \(\min_s n_s>0\).
The gradient records how an objective changes under a small change in its
quantile functions. With respect to the \(r\)-weighted inner product, these
gradients, formally the Fr\'echet gradients, are
\[
[\nabla G(\mathbf Q)]_s(u)=\Psi(Q_s^*(u),Q_s(u)),
\qquad
[\nabla \widehat G_{n,\hat r}^C(\mathbf Q)]_s(u)=\frac{\hat r_s}{r_s}\Psi(\hat Q_s^*(u),Q_s(u)),
\]
where $\Psi(z,q):=\partial_q\mathcal C(z,q)$.

\paragraph{Step 2: Use optimality within the fairness constraints.}
Let
\[
\mathbf Q^{\mathrm{fair}}\in\argmin_{\mathbf Q\in\mathcal K}G(\mathbf Q),
\qquad
\hat{\mathbf Q}\in\argmin_{\mathbf Q\in\mathcal K}\widehat G_{n,\hat r}^C(\mathbf Q),
\]
and set \(\Delta:=\hat{\mathbf Q}-\mathbf Q^{\mathrm{fair}}\). Since both
optimizers lie in the same convex set, moving from either optimizer toward any
other feasible point cannot lower its objective. These first-order conditions
are called variational inequalities
\citep{showalter2009hilbert}:
\[
\langle \nabla G(\mathbf Q^{\mathrm{fair}}),\mathbf Q-\mathbf Q^{\mathrm{fair}}\rangle\ge 0
\quad\text{for all }\mathbf Q\in\mathcal K,
\qquad
\langle \nabla \widehat G_{n,\hat r}^C(\hat{\mathbf Q}),\mathbf Q-\hat{\mathbf Q}\rangle\ge 0
\quad\text{for all }\mathbf Q\in\mathcal K.
\]
Take \(\mathbf Q=\hat{\mathbf Q}\) in the first inequality. Take
\(\mathbf Q=\mathbf Q^{\mathrm{fair}}\) in the second. This gives
\[
\langle \nabla G(\mathbf Q^{\mathrm{fair}}),\Delta\rangle\ge 0,
\qquad
\langle \nabla \widehat G_{n,\hat r}^C(\hat{\mathbf Q}),\Delta\rangle\le 0.
\]
By \(\mu\)-strong convexity of \(G\),
\[
\mu\|\Delta\|_2^2
\le \langle \nabla G(\hat{\mathbf Q})-\nabla G(\mathbf Q^{\mathrm{fair}}),\Delta\rangle
\le \langle \nabla G(\hat{\mathbf Q}),\Delta\rangle,
\]
where the last inequality uses
\(\langle \nabla G(\mathbf Q^{\mathrm{fair}}),\Delta\rangle\ge0\).
Next, add and subtract
\(\nabla\widehat G_{n,\hat r}^C(\hat{\mathbf Q})\). Using
\(\langle \nabla \widehat G_{n,\hat r}^C(\hat{\mathbf Q}),\Delta\rangle\le0\)
then yields
\begin{equation}\label{eq:key_pairing_final}
\mu\|\Delta\|_2^2
\le \big\langle \nabla G(\hat{\mathbf Q})-\nabla \widehat G_{n,\hat r}^C(\hat{\mathbf Q}),\Delta\big\rangle.
\end{equation}
This inequality is useful because its left side measures the squared distance
between the two optimizers. Its right side depends only on the difference
between the population and estimated gradients.

\paragraph{Step 3: Bound the difference between the gradients.}
The Cauchy--Schwarz inequality separates the gradient error from the optimizer
error:
\[
|\langle \nabla G(\hat{\mathbf Q})-\nabla \widehat G_{n,\hat r}^C(\hat{\mathbf Q}),\Delta\rangle|
\le
\|\nabla G(\hat{\mathbf Q})-\nabla \widehat G_{n,\hat r}^C(\hat{\mathbf Q})\|_2\,\|\Delta\|_2,
\]
while Assumption~\ref{assum:cost_regularity}(2) gives
\begin{align*}
\|\nabla G(\hat{\mathbf Q})-\nabla \widehat G_{n,\hat r}^C(\hat{\mathbf Q})\|_2^2
&=
\sum_{s=1}^K r_s\int_0^1
\left\{
\Psi(Q_s^*(u),\hat Q_s(u))
-
\frac{\hat r_s}{r_s}\Psi(\hat Q_s^*(u),\hat Q_s(u))
\right\}^2du\\
&\le
2L_z^2\sum_{s=1}^K r_s\int_0^1
\left\{Q_s^*(u)-\hat Q_s^*(u)\right\}^2du\\
&\quad
+2\sum_{s=1}^K r_s\left(\frac{\hat r_s-r_s}{r_s}\right)^2
\int_0^1\Psi(\hat Q_s^*(u),\hat Q_s(u))^2du\\
&\le
2L_z^2\|\hat{\mathbf Q}^{\,*}-\mathbf Q^*\|_2^2
+O_p\!\left(\max_{1\le s\le K}|\hat r_s-r_s|^2\right)
\end{align*}
where the last line uses Assumption~\ref{assum:group_proportions} and the derived score-gradient bound following Assumption~\ref{assum:cost_regularity}.
The first term comes from estimating the score quantiles. The second comes
from estimating the group weights. Therefore,
\[
\big\langle \nabla G(\hat{\mathbf Q})-\nabla \widehat G_{n,\hat r}^C(\hat{\mathbf Q}),\Delta\big\rangle
\le
O_p\!\left(\|\hat{\mathbf Q}^{\,*}-\mathbf Q^*\|_2+\max_{1\le s\le K}|\hat r_s-r_s|\right)\|\Delta\|_2
\]
with probability tending to one.

\paragraph{Step 4: Insert the two input-error rates.}
Insert the preceding bound into \eqref{eq:key_pairing_final}. On the event
\(\Delta\neq0\), divide both sides by \(\|\Delta\|_2\).
Lemma~\ref{lem:latent_convergence} and
Assumption~\ref{assum:group_proportions} then give
\[
\|\hat{\mathbf Q}-\mathbf Q^{\mathrm{fair}}\|_2
\le
O_p\!\left(\|\hat{\mathbf Q}^{\,*}-\mathbf Q^*\|_2+\max_{1\le s\le K}|\hat r_s-r_s|\right)
=O_p(\rho_m+n^{-1/2}),
\]
and squaring yields
\[
\sum_{s=1}^K r_s\|\hat Q_s-Q_s^{\mathrm{fair}}\|_2^2
=O_p(\rho_m^2+n^{-1}),
\]
which proves the theorem.
\end{proof}

\begin{proof}[Proof of Lemma~\ref{lem:rank_stability}]
The estimated rank uses both a fitted score and an empirical distribution
function. We first control the sampling error in that empirical distribution
function. We then compare the fitted-score distribution with the true-score
distribution. Finally, we show that the resulting ranks are nearly uniform,
as true population ranks are.

\emph{Sampling error in the calibration distribution.}
Fix a group $s$ with $r_s>0$. Conditional on the auxiliary training sample,
$\hat f(\cdot,s)$ is fixed. Let
\[
F_{\hat f,s}(t)
:=
\mathbb P\{\hat f(X^s,s)\le t\mid \mathcal D_m^{\mathrm{tr}}\}
\]
denote the conditional distribution function of the fitted score in group $s$.
The rank error is at most the sum of the empirical-distribution error and the
error caused by replacing the true score with the fitted score:
\begin{align*}
|\hat F_s(\hat f(X^s,s))-F_s^*(f^*(X^s,s))|
&\le
\sup_t|\hat F_s(t)-F_{\hat f,s}(t)| \\
&\quad+
|F_{\hat f,s}(\hat f(X^s,s))-F_s^*(f^*(X^s,s))|.
\end{align*}
The first term is the usual sampling error of an empirical cumulative
distribution function. Conditional on $\mathcal D_m^{\mathrm{tr}}$, the
Dvoretzky--Kiefer--Wolfowitz inequality gives
\[
\sup_t|\hat F_s(t)-F_{\hat f,s}(t)|=O_p(n_s^{-1/2})=o_p(1).
\]

\emph{Difference between fitted and true score distributions.}
For any $\eta>0$, a fitted and true score can fall on different sides of a
threshold \(t\) only if their difference exceeds \(\eta\) or the true score
lies within \(\eta\) of \(t\). Therefore
\[
\sup_t|F_{\hat f,s}(t)-F_s^*(t)|
\le
\mathbb P\{|\hat f(X^s,s)-f^*(X^s,s)|>\eta
\mid \mathcal D_m^{\mathrm{tr}}\}
+
\sup_t\mathbb P\{|f^*(X^s,s)-t|\le \eta\}.
\]
By Assumption~\ref{assum:density}, the density of $f^*(X^s,s)$ is bounded above,
so little probability can lie in an interval of length \(2\eta\):
\[
\sup_t\mathbb P\{|f^*(X^s,s)-t|\le \eta\}
\le 2\bar c\,\eta.
\]
By Assumption~\ref{assum:first_step}, the first probability on the right-hand
side tends to zero if $\eta=\eta_n\downarrow 0$ slowly enough. Hence
\[
\sup_t|F_{\hat f,s}(t)-F_s^*(t)|=o_p(1).
\]
The upper density bound also means that $F_s^*$ is Lipschitz with constant at
most $\bar c$. Thus
\[
|F_s^*(\hat f(X^s,s))-F_s^*(f^*(X^s,s))|
\le
\bar c\,|\hat f(X^s,s)-f^*(X^s,s)|,
\]
whose conditional second moment is $o_p(1)$ by
Assumption~\ref{assum:first_step}. Combining these bounds gives
\[
\mathbb E\!\left[
\left\{
\hat F_s(\hat f(X^s,s))-F_s^*(f^*(X^s,s))
\right\}^2
\mid \mathcal H_n
\right]
=o_p(1).
\]

\emph{The estimated ranks are nearly uniform.}
Because \(X^s\) is independent of \(\mathcal H_n\)
and \(F_s^*\) is continuous under Assumption~\ref{assum:density},
\(U_s^*\sim\operatorname{Unif}(0,1)\) conditionally on \(\mathcal H_n\). Let
\[
e_{s,n}
:=
\mathbb E\!\left[
\left\{\hat U_s-U_s^*\right\}^2
\mid \mathcal H_n
\right].
\]
The preceding conditional \(L_2\) bound and fixed \(K\) imply
\[
\max_{s:r_s>0}e_{s,n}=o_p(1).
\]
If \(\hat U_s\) is within \(\delta\) of \(U_s^*\), then the event
\(\hat U_s\le u\) is trapped between the events \(U_s^*\le u-\delta\) and
\(U_s^*\le u+\delta\). Markov's inequality controls the remaining event.
For any \(\delta\in(0,1)\) and \(u\in[0,1]\), this gives
\[
H_{s,n}(u)
\le
\mathbb P(U_s^*\le u+\delta\mid\mathcal H_n)
+
\mathbb P(|\hat U_s-U_s^*|>\delta\mid\mathcal H_n)
\le
u+\delta+\delta^{-2}e_{s,n},
\]
and similarly
\[
H_{s,n}(u)
\ge
\mathbb P(U_s^*\le u-\delta\mid\mathcal H_n)
-
\mathbb P(|\hat U_s-U_s^*|>\delta\mid\mathcal H_n)
\ge
u-\delta-\delta^{-2}e_{s,n},
\]
with the obvious truncation at the endpoints. Hence
\[
\sup_{u\in[0,1]}|H_{s,n}(u)-u|
\le
\delta+\delta^{-2}e_{s,n}.
\]
Because the number of groups is fixed, we may take the maximum over groups.
Choosing a deterministic \(\delta=\delta_n\downarrow0\) slowly enough that
\(\delta_n^{-2}\max_{s:r_s>0}e_{s,n}=o_p(1)\) gives
\[
\max_{s:r_s>0}\sup_{u\in[0,1]}|H_{s,n}(u)-u|=o_p(1).
\]
\end{proof}

\begin{proof}[Proof of Corollary~\ref{cor:ffair_consistency}]
The final prediction evaluates an estimated curve at an estimated rank. Its
error therefore has two parts: error in the fair quantile curve and error in
the rank used as its input. We show that both parts vanish.

Fix a group $s$ with $r_s>0$ and write
\[
U_s^*=F_s^*(f^*(X^s,s)),
\qquad
\hat U_s=\hat F_s(\hat f(X^s,s)).
\]
Insert and subtract \(Q_s^{\mathrm{fair}}(\hat U_s)\). The triangle inequality
then separates the two errors:
\begin{align*}
&\left\{
\mathbb E\left[
\left\{
\hat f_{\mathrm{fair}}(X^s,s)
-
f_{\mathrm{fair}}^*(X^s,s)
\right\}^2
\mid\mathcal H_n
\right]
\right\}^{1/2} \\
&\quad\le
\left\{
\mathbb E\left[
\left\{
\hat Q_s(\hat U_s)
-
Q_s^{\mathrm{fair}}(\hat U_s)
\right\}^2
\mid\mathcal H_n
\right]
\right\}^{1/2} \\
&\qquad+
\left\{
\mathbb E\left[
\left\{
Q_s^{\mathrm{fair}}(\hat U_s)
-
Q_s^{\mathrm{fair}}(U_s^*)
\right\}^2
\mid\mathcal H_n
\right]
\right\}^{1/2}.
\end{align*}

\emph{Error from using an estimated rank.}
Let \(\mathcal I\) be the compact feasible prediction interval from
Assumption~\ref{assum:cost_regularity}, and write
\[
D_{\mathcal I}:=\sup_{a,b\in\mathcal I}|a-b|<\infty.
\]
At scale \(\delta\), the modulus of continuity of the population fair
quantile \(Q_s^{\mathrm{fair}}\) is
\[
\omega_s(\delta)
:=
\sup_{|u-v|\le\delta}
|Q_s^{\mathrm{fair}}(u)-Q_s^{\mathrm{fair}}(v)|.
\]
By Assumption~\ref{assum:continuous_fair_quantiles},
\(\omega_s(\delta)\downarrow0\) as \(\delta\downarrow0\). For every
\(\delta\in(0,1)\), the curve changes by at most \(\omega_s(\delta)\) when the
two ranks are close. When they are farther apart, its bounded range gives
\[
\left|Q_s^{\mathrm{fair}}(\hat U_s)-Q_s^{\mathrm{fair}}(U_s^*)\right|
\le
\omega_s(\delta)
+
D_{\mathcal I}\mathbbm 1\{|\hat U_s-U_s^*|>\delta\}.
\]
Therefore,
\[
\mathbb E\!\left[
\left\{
Q_s^{\mathrm{fair}}(\hat U_s)-Q_s^{\mathrm{fair}}(U_s^*)
\right\}^2
\mid\mathcal H_n
\right]
\le
2\omega_s(\delta)^2
+
2D_{\mathcal I}^2\delta^{-2}
\mathbb E\!\left[
\left\{\hat U_s-U_s^*\right\}^2
\mid\mathcal H_n
\right].
\]
For fixed \(\delta\), Lemma~\ref{lem:rank_stability} sends the second term to
zero as \(n\to\infty\). We can then let \(\delta\downarrow0\), which sends the
first term to zero. Thus the rank-input error is \(o_p(1)\).

\emph{Error from estimating the fair quantile curve.}
Let
\[
h_s=\hat Q_s-Q_s^{\mathrm{fair}},
\qquad
g_s(u)=h_s(u)^2,
\qquad
H_{s,n}(u)=\mathbb P(\hat U_s\le u\mid\mathcal H_n).
\]
Conditional on \(\mathcal H_n\),
\[
\mathbb E\!\left[
\left\{
\hat Q_s(\hat U_s)-Q_s^{\mathrm{fair}}(\hat U_s)
\right\}^2
\mid\mathcal H_n
\right]
=
\int_0^1 g_s(u)\,dH_{s,n}(u).
\]
Both \(\hat Q_s\) and \(Q_s^{\mathrm{fair}}\) are nondecreasing quantile
functions with values in \(\mathcal I\). Their squared difference is therefore
bounded and cannot oscillate without limit. More precisely,
\(\|g_s\|_\infty\le D_{\mathcal I}^2\) and
\(\operatorname{TV}(g_s)\le 4D_{\mathcal I}^2\), where
\(\operatorname{TV}\) denotes total variation. Stieltjes integration by parts
\citep{weikard2018real} then compares the average under the estimated rank
distribution with the average under a uniform rank:
\begin{align*}
\left|
\int_0^1 g_s(u)\,dH_{s,n}(u)
-
\int_0^1 g_s(u)\,du
\right|
&\le
\left\{\operatorname{TV}(g_s)+2\|g_s\|_\infty\right\}
\sup_{u\in[0,1]}|H_{s,n}(u)-u|\\
&\le
6D_{\mathcal I}^2
\sup_{u\in[0,1]}|H_{s,n}(u)-u|.
\end{align*}
The supremum term includes the contributions from both endpoints. By
Lemma~\ref{lem:rank_stability}, the right-hand side is \(o_p(1)\). The uniform
rank average is
\[
\int_0^1 g_s(u)\,du
=
\|\hat Q_s-Q_s^{\mathrm{fair}}\|_2^2.
\]
Theorem~\ref{thm:consistency}, fixed \(K\), and
Assumption~\ref{assum:group_proportions} imply this is \(o_p(1)\) for each
group with \(r_s>0\). Thus the curve-estimation error is also \(o_p(1)\).

Combining the two bounds and summing over the finitely many groups with weights
\(r_s\) yields
\[
\mathcal E_f(\hat f_{\mathrm{fair}},f_{\mathrm{fair}}^*)=o_p(1),
\]
which proves the corollary.
\end{proof}

\subsection{Derivation of Closed-Form Solutions} \label{app:derivations}

This subsection derives the formulas in Section~\ref{sec:closed_form}. Assume
$r_1,r_2>0$ and $r_1+r_2=1$. The key idea is to describe the two quantile
functions by their weighted average and their difference. The fairness
constraint acts only on the difference. The average can therefore remain at
its unconstrained value, while the difference is moved to the nearest
feasible function.

For \(Q_1,Q_2\in L_2([0,1])\), the objective is
\begin{equation} \label{eq:gen_prob}
    \min_{Q_1, Q_2} \mathcal{J}(Q_1, Q_2) := \int_0^1 \left[ r_1 (Q_1^*(t) - Q_1(t))^2 + r_2 (Q_2^*(t) - Q_2(t))^2 \right] dt,
\end{equation}
subject to the distance constraint $\|Q_1 - Q_2\|_p \le \epsilon$.

\subsubsection*{Decomposition of the Problem}
Let $\bar{Q}(t) = r_1 Q_1(t) + r_2 Q_2(t)$ be the weighted average and
$\Delta(t) = Q_1(t) - Q_2(t)$ the difference function. This change of
variables keeps all information because we can recover the two original
functions from
$$ Q_1(t) = \bar{Q}(t) + r_2 \Delta(t), \quad Q_2(t) = \bar{Q}(t) - r_1 \Delta(t). $$
Define $\bar{Q}^*(t)$ and $\Delta^*(t)$ in the same way for the latent
quantiles. Substitution into \eqref{eq:gen_prob} gives
\begin{align*}
    Q_1^* - Q_1 &= (\bar{Q}^* + r_2 \Delta^*) - (\bar{Q} + r_2 \Delta) = (\bar{Q}^* - \bar{Q}) + r_2(\Delta^* - \Delta), \\
    Q_2^* - Q_2 &= (\bar{Q}^* - r_1 \Delta^*) - (\bar{Q} - r_1 \Delta) = (\bar{Q}^* - \bar{Q}) - r_1(\Delta^* - \Delta).
\end{align*}
Expanding the two squares gives
\begin{align*}
    \mathcal{J} &= \int r_1 \left[ (\bar{Q}^* - \bar{Q}) + r_2(\Delta^* - \Delta) \right]^2 + r_2 \left[ (\bar{Q}^* - \bar{Q}) - r_1(\Delta^* - \Delta) \right]^2 dt \\
    &= \int (\bar{Q}^* - \bar{Q})^2 (r_1 + r_2) + (\Delta^* - \Delta)^2 (r_1 r_2^2 + r_2 r_1^2) dt.
\end{align*}
The terms that mix the average and the difference cancel. Since
$r_1+r_2=1$, the remaining coefficients simplify to
$$ \mathcal{J}(\bar{Q}, \Delta) = \|\bar{Q}^* - \bar{Q}\|_2^2 + r_1 r_2 \|\Delta^* - \Delta\|_2^2. $$
Since the constraint $\|Q_1 - Q_2\|_p \le \epsilon$ depends only on $\Delta$
(specifically, $\|\Delta\|_p \le \epsilon$), the two parts can be minimized
separately.
\begin{enumerate}
    \item Minimizing over $\bar Q$ gives $\bar Q=\bar Q^*$. Thus the adjustment preserves the weighted group average.
    \item Minimizing over $\Delta$ means finding the feasible difference closest to $\Delta^*$ in $L_2$. This is a projection problem.
\end{enumerate}

\subsubsection*{Case 1: $L_2$-Constraint ($p=2$)}
Here we seek the function inside the \(L_2\) ball that is closest to
\(\Delta^*\). In symbols, we minimize
$\frac{1}{2} \|\Delta^* - \Delta\|_2^2$ subject to
$\|\Delta\|_2^2 \le \epsilon^2$. This operation is projection onto a closed
convex ball in $L_2([0,1])$ \citep{bauschkecombettes2017}.
If $\epsilon=0$, the feasible set is the singleton $\Delta=0$, which is
the solution. Now suppose $\epsilon>0$. Let \(\lambda\ge0\) be the Lagrange
multiplier for the radius constraint. It can be positive only when that
constraint is active. The Lagrangian is
$$ \mathfrak{L}(\Delta, \lambda) = \frac{1}{2} \int_0^1 (\Delta^*(t) - \Delta(t))^2 dt + \frac{\lambda}{2} \left( \int_0^1 \Delta(t)^2 dt - \epsilon^2 \right). $$
Setting its derivative with respect to \(\Delta\) equal to zero gives
$$ -(\Delta^*(t) - \Delta(t)) + \lambda \Delta(t) = 0 \implies \Delta(t) = \frac{\Delta^*(t)}{1+\lambda}. $$
If $\|\Delta^*\|_2 \le \epsilon$, the original difference is already feasible,
so $\lambda=0$ and $\Delta = \Delta^*$. If
$\|\Delta^*\|_2 > \epsilon$, the solution lies on the boundary. Substituting
$\Delta$ into $\|\Delta\|_2 = \epsilon$ yields
$\frac{1}{1+\lambda}\|\Delta^*\|_2 = \epsilon$, and therefore
$\frac{1}{1+\lambda} = \frac{\epsilon}{\|\Delta^*\|_2}$. The solution is
linear shrinkage:
\[
\Delta(t)=
\begin{cases}
\Delta^*(t),&\|\Delta^*\|_2\le\epsilon,\\
\dfrac{\epsilon}{\|\Delta^*\|_2}\Delta^*(t),
&\|\Delta^*\|_2>\epsilon.
\end{cases}
\]
The second branch has a strictly positive denominator; in particular, when
\(\epsilon=\|\Delta^*\|_2=0\), the first branch applies.
When the constraint is active, the same factor reduces the group difference
at every rank.

\subsubsection*{Case 2: $L_1$-Constraint ($p=1$)}
We now minimize $\frac{1}{2} \|\Delta^* - \Delta\|_2^2$ subject to
$\|\Delta\|_1 \le \epsilon$. Let \(\lambda\ge0\) be the Lagrange multiplier
for this constraint. The Lagrangian is
$$ \mathfrak{L}(\Delta, \lambda) = \frac{1}{2} \int_0^1 (\Delta^*(t) - \Delta(t))^2 dt + \lambda \left( \int_0^1 |\Delta(t)| dt - \epsilon \right). $$
The absolute-value function has a corner at zero, so it does not have an
ordinary derivative there. Its subgradient is
\(\partial|x|=\{\operatorname{sgn}(x)\}\) when \(x\ne0\) and
\(\partial|0|=[-1,1]\). The optimality condition says that, for each \(t\),
there is a \(v(t)\in\partial|\Delta(t)|\) such that
$$ -(\Delta^*(t) - \Delta(t)) + \lambda v(t) = 0 \implies \Delta(t) + \lambda v(t) = \Delta^*(t). $$
This equation gives the soft-thresholding rule.
\begin{itemize}
    \item If $|\Delta^*(t)| > \lambda$, then $\Delta(t) \ne 0$ and
    $v(t) = \operatorname{sgn}(\Delta(t))$. It follows that
    $\Delta(t) = \Delta^*(t) - \lambda \operatorname{sgn}(\Delta^*(t))$.
    \item If $|\Delta^*(t)| \le \lambda$ and $\lambda>0$, setting
    $\Delta(t) = 0$ allows $v(t) = \Delta^*(t)/\lambda \in [-1, 1]$,
    satisfying the condition. On the inactive branch $\lambda=0$, the only
    point in this case has $\Delta^*(t)=0$, and one may take $v(t)=0$.
\end{itemize}
Thus, $\Delta(t) = \operatorname{sgn}(\Delta^*(t)) \max(0, |\Delta^*(t)| - \lambda)$.
If $\epsilon\ge\|\Delta^*\|_1$, the original difference is feasible, so take
$\lambda=0$ and $\Delta=\Delta^*$. For
$0<\epsilon<\|\Delta^*\|_1$, the value of $\lambda$ is uniquely determined by
$\int |\Delta(t)|dt=\epsilon$. For $\epsilon=0$, the solution is $\Delta=0$.
Unlike the $L_2$ rule, this projection can remove small local differences
entirely while retaining larger ones after a common threshold is subtracted.

\subsubsection*{Case 3: $L_\infty$-Constraint ($p=\infty$)}
We minimize $\frac{1}{2} \int (\Delta^*(t) - \Delta(t))^2 dt$ subject to
$|\Delta(t)| \le \epsilon$ for almost all $t$. Both the loss and the
constraint can be considered one rank at a time. At each \(t\), we choose the
number in \([-\epsilon,\epsilon]\) closest to \(\Delta^*(t)\). This is
clipping:
$$ \Delta(t) = \begin{cases} \epsilon & \text{if } \Delta^*(t) > \epsilon, \\ \Delta^*(t) & \text{if } |\Delta^*(t)| \le \epsilon, \\ -\epsilon & \text{if } \Delta^*(t) < -\epsilon. \end{cases} $$
Thus only differences outside the admissible fairness band are changed.

\subsection{Squared loss with \texorpdfstring{$K>2$}{K greater than 2}}
We now specialize problem~\eqref{Final-Opt} to squared loss with $K\ge 2$
groups. With more than two groups,
there is no single difference function that describes every pair. Even so,
two parts of the two-group geometry remain. The adjustment preserves the
weighted average of the group quantiles. Also, under \(W_\infty\), each rank
can still be solved separately.
The fair quantiles $(Q_1,\dots,Q_K)$ solve
\begin{equation}\label{eq:Kgroup_squared}
    \min_{Q_1,\dots,Q_K\in\mathcal Q_{\mathcal I}} \ \sum_{s=1}^K r_s \|Q_s - Q_s^*\|_2^2
    \quad \text{s.t.}\quad \|Q_i - Q_j\|_p \le \epsilon,\quad \forall i,j\in\mathcal{S}.
\end{equation}

\paragraph{A key structural fact (multi-group analog of the two-group barycenter identity).}
At each rank, define the weighted averages before and after adjustment by
\[
    \bar Q(u) := \sum_{s=1}^K r_s Q_s(u),
    \qquad
    \bar Q^*(u) := \sum_{s=1}^K r_s Q_s^*(u).
\]
The following result shows that they are equal for any $p\in[1,\infty]$.

\begin{proposition}[Weighted barycenter preservation under squared loss]\label{prop:barycenter_K}
Suppose $r_s>0$ and $Q_s^*\in\mathcal Q_{\mathcal I}$. The optimizer of
\eqref{eq:Kgroup_squared} is unique and satisfies
\begin{equation}\label{eq:barycenter_preserve}
    \sum_{s=1}^K r_s \widehat Q_s(u)
    \;=\;
    \sum_{s=1}^K r_s Q_s^*(u),
    \qquad \text{for almost all } u\in(0,1).
\end{equation}
\end{proposition}

\begin{proof}
The main idea is that shifting every group curve by the same function does not
change any pairwise distance. At the optimum, such a shared shift cannot
reduce squared loss. This forces the weighted average to match its original
value. We first make this argument without the quantile and range
requirements. We then show that the unique solution already satisfies those
requirements.

Consider the weighted space $L_2(0,1)^K$ and impose only the pairwise
constraints. This is a nonempty closed convex set, so the closest point
exists. Strict convexity makes that point unique; denote it by
\(\widetilde{\mathbf Q}\). Adding the same function \(h\) to every component
leaves every group difference unchanged. Therefore, the objective cannot
decrease under any such positive or negative shared shift. Its slope in this
direction must be zero, which gives
\[
\sum_{s=1}^K r_s(\widetilde Q_s-Q_s^*)=0
\]
in $L_2$.

We next check the range. Clip every component to \(\mathcal I\). Clipping
cannot enlarge a distance because it is \(1\)-Lipschitz. Since every target
\(Q_s^*\) lies in \(\mathcal I\), clipping also cannot increase squared loss.
The minimizer is unique, so clipping cannot change it. Thus every
\(\widetilde Q_s\) already lies in \(\mathcal I\).

Finally, sort each component into quantile order. By
Lemma~\ref{lem:rearrangement}, this cannot increase the pairwise distances or
the squared loss. Uniqueness again means that sorting cannot change the
minimizer. Hence its components are already nondecreasing and left-continuous.
It is therefore the unique optimizer of
\eqref{eq:Kgroup_squared}, and the shared-shift identity above gives
\eqref{eq:barycenter_preserve}.
\end{proof}

\paragraph{Closed-form characterization when $p=\infty$.}
Under \(W_\infty\), no two group quantiles may differ by more than
\(\epsilon\) at any rank:
\[
    \|Q_i-Q_j\|_\infty \le \epsilon
    \quad\Longleftrightarrow\quad
    |Q_i(u)-Q_j(u)|\le \epsilon\ \ \text{for almost all }u\in(0,1),\ \forall i,j\in\mathcal{S}.
\]
Both the squared loss and these constraints can be handled one rank \(u\) at
a time. At a fixed rank, we solve
\[
    \min_{q_1,\dots,q_K\in\mathcal I} \ \sum_{s=1}^K r_s (q_s-Q_s^*(u))^2
    \quad \text{s.t.}\quad |q_i-q_j|\le \epsilon,\ \forall i,j\in\mathcal{S}.
\]
Equivalently, the constraint is
\(\max_s q_s-\min_s q_s\le\epsilon\): all \(K\) adjusted values must fit
inside one interval of width \(\epsilon\). We choose the best such interval
and clip each original group value to it. When $0\le\epsilon<b-a$, let
$a(u)\in[a,b-\epsilon]$ be its lower endpoint. Then
\begin{equation}\label{eq:pinfty_clip}
    \widehat Q_s(u)
    =
    \min\{a(u)+\epsilon,\ \max\{a(u),\ Q_s^*(u)\}\},
    \qquad s=1,\dots,K,
\end{equation}
where \(a(u)\) minimizes the squared loss at that rank. Equivalently, it is
chosen so that the clipping preserves the weighted average:
$\sum_{s=1}^K r_s \widehat Q_s(u)=\sum_{s=1}^K r_s Q_s^*(u)$ for almost all
$u$. The multi-group rule is therefore explicit except for a
one-dimensional search for the interval endpoint. If $\epsilon\ge b-a$, all
original group values already fit inside an allowed interval.

\paragraph{Why closed forms generally do not extend to $p\in\{1,2\}$ when $K>2$.}
For $p<\infty$, each constraint $\|Q_i-Q_j\|_p\le \epsilon$ combines
differences across the full rank interval \((0,1)\). We can no longer solve
problem~\eqref{eq:Kgroup_squared} one rank at a time. We must also satisfy all
\(\binom{K}{2}\) pairwise constraints together. In general, one common
shrinkage factor or separate groupwise thresholding rules cannot meet all of
these requirements. This is why the paper uses numerical projection after
placing the quantiles on a finite grid.

\paragraph{Why the $p=2$ formula becomes harder with more than two groups.}
The \(p=2\) first-order conditions show exactly where the difficulty appears.
When $\epsilon>0$, write the constraints in squared form:
$g_{ij}(Q):=\|Q_i-Q_j\|_2^2-\epsilon^2\le 0$ for all $i<j$.
For each pair, let \(\lambda_{ij}\ge0\) be the
Karush--Kuhn--Tucker (KKT) multiplier for the constraint \(g_{ij}(Q)\le0\).
A positive multiplier identifies a constraint that affects the first-order
conditions. Write
$\lambda_{sj}:=\lambda_{\min\{s,j\},\max\{s,j\}}$ when $s\ne j$. Then,
for each $s\in\mathcal{S}$,
\begin{equation}\label{eq:KKT_p2_stationarity}
    r_s(\widehat Q_s-Q_s^*)
    +\sum_{j\neq s}\lambda_{sj}(\widehat Q_s-\widehat Q_j)
    =0,
\end{equation}
together with feasibility of the quantiles and multipliers. The complementary
slackness condition
$\lambda_{ij}\big(\|\widehat Q_i-\widehat Q_j\|_2^2-\epsilon^2\big)=0$
means that a multiplier can be positive only when its pairwise constraint is
exactly at the tolerance.

If the multipliers were known, \eqref{eq:KKT_p2_stationarity} would be linear
in $(\widehat Q_1,\dots,\widehat Q_K)$. They are not known, however, because
we do not know in advance which distances satisfy
$\|\widehat Q_i-\widehat Q_j\|_2=\epsilon$. With two groups, there is only one
multiplier, which gives the single \(L_2\) shrinkage factor. With more groups,
as many as \(\binom{K}{2}\) multipliers must be determined together.
The equations must therefore identify both the solution and the constraints
that are active. This coupled problem is called a nonlinear complementarity
system. It generally requires a numerical solver. Examples include a method
that searches over active constraints, Dykstra-type projection, the
alternating direction method of multipliers (ADMM), and projected gradient on
a finite grid.

When $\epsilon=0$, the appropriate constraints are the linear equalities
$Q_i=Q_j$ almost everywhere. This is the strict demographic-parity
barycenter problem. The squared-constraint equations above are not used when
the radius is zero.

\paragraph{A simpler sufficient rule for multiple groups.}
The exact pairwise problem usually needs a numerical solver. A simpler,
conservative alternative compares every group with the same fixed weighted
average
$\bar Q^* := \sum_{s=1}^K r_s Q_s^*$. It is conservative because its
requirements are stronger than the original pairwise requirements, so it may
exclude some feasible solutions. Requiring each group to stay within
\(\epsilon/2\) of this average gives
\begin{equation}\label{eq:Kgroup_squared_barycenter_ball}
    \min_{Q_1,\dots,Q_K\in\mathcal Q_{\mathcal I}} \ \sum_{s=1}^K r_s \|Q_s - Q_s^*\|_2^2
    \quad \text{s.t.}\quad \|Q_s-\bar Q^*\|_p \le \epsilon/2,\ \ s=1,\dots,K.
\end{equation}
Writing $\delta_s:=Q_s-\bar Q^*$ and
$\delta_s^*:=Q_s^*-\bar Q^*$, problem
\eqref{eq:Kgroup_squared_barycenter_ball} can be solved separately for each
group:
\[
    \min_{\delta_1,\dots,\delta_K}\ \sum_{s=1}^K r_s \|\delta_s-\delta_s^*\|_2^2
    \quad \text{s.t.}\quad \|\delta_s\|_p\le \epsilon/2,\ \ s=1,\dots,K,
\]
so each group uses the corresponding two-group projection formula. For
$p=2$, use $\widehat\delta_s=\delta_s^*$ when
$\|\delta_s^*\|_2\le\epsilon/2$ and otherwise use
$\widehat\delta_s=\{(\epsilon/2)/\|\delta_s^*\|_2\}\delta_s^*$. For $p=1$,
soft-threshold to radius $\epsilon/2$. For $p=\infty$, clip pointwise to
$[-\epsilon/2,\epsilon/2]$. In every case, set
$\widehat Q_s=\bar Q^*+\widehat\delta_s$.
Each update of \(\delta_s^*\) preserves order, never expands differences, and
keeps zero fixed. In the terminology used earlier, it is a nondecreasing
\(1\)-Lipschitz map fixing zero. Compare its change between any two ranks, as
in the proof of Lemma~\ref{lem:monotonicity_closed_form}, using \(Q_s^*\) and
\(\bar Q^*\). This shows that \(\widehat Q_s\) is nondecreasing and
left-continuous. It lies between these two functions at every rank and
therefore remains in \(\mathcal I\).

Different groups may use different shrinkage factors, so this simpler rule
need not preserve \(\bar Q^*\). It does, however, guarantee the original
pairwise tolerance. The triangle inequality gives
$\|\widehat Q_i-\widehat Q_j\|_p \le
\|\widehat Q_i-\bar Q^*\|_p+\|\widehat Q_j-\bar Q^*\|_p \le \epsilon$.
Thus each group needs radius $\epsilon/2$, rather than $\epsilon$, to ensure
that every pair is at most \(\epsilon\) apart.

\end{document}